\documentclass{article}

\usepackage{amssymb,amsmath,amsthm,color,turnstile,txfonts,mathrsfs,bm}
\usepackage{graphicx}
\usepackage{epstopdf}
\usepackage{url,hyperref}
\usepackage{multirow}
\usepackage[left=1in,top=1in,right=1in,bottom=1in,letterpaper]{geometry}
\usepackage{setspace}
 
\usepackage[ruled,vlined]{algorithm2e}
\usepackage{wrapfig}
\usepackage{enumitem}
\usepackage{cite,caption}
\usepackage{amsmath,amsfonts,bm}

\newtheorem{theorem}{Theorem}[section]

\newtheorem{definition}[theorem]{Definition}

\usepackage{url}
\usepackage[utf8]{inputenc} 
\usepackage[T1]{fontenc}    
\usepackage{url}            
\usepackage{booktabs}       
\usepackage{amsfonts}       
\usepackage{nicefrac}       
\usepackage{microtype}      
\usepackage{xcolor}         
\usepackage{enumitem}
\setlist[itemize]{itemsep=1pt,leftmargin=5.5mm}
\usepackage{lipsum}
\usepackage{amsfonts}
\usepackage{graphicx}
\usepackage{epstopdf}
\usepackage{algorithmic}
\ifpdf
  \DeclareGraphicsExtensions{.eps,.pdf,.png,.jpg}
\else
  \DeclareGraphicsExtensions{.eps}
\fi
\usepackage{float}
\usepackage{amsmath}
\usepackage{blkarray}
\usepackage{bm}
\usepackage{color}
\usepackage{esvect}
\usepackage[utf8]{inputenc}
\usepackage{makecell}
\usepackage{soul}
\usepackage{tikz}
\usepackage{multirow}
\usepackage{caption}
\usepackage{subcaption}

\def\M{\mathcal{M}}
\def\H{\mathcal{H}}

\def\R{\mathcal{R}}
\def\RR{\mathbb{R}}
\def\S{\mathcal{S}}

\def\x{\mathbf{x}}

\def\p{\mathbf{p}}

\def\real{\mathbb{R}}

\DeclareMathOperator*{\esssup}{ess\,sup}
\DeclareMathOperator*{\argmin}{arg\,min}
\DeclareMathOperator*{\Med}{Med}
\DeclareMathOperator*{\argmax}{arg\,max}

\usepackage{authblk}

\title{Semi-Supervised Manifold Learning  with Complexity Decoupled Chart Autoencoders}

\author[1]{Stefan C. Schonsheck}
\affil[1]{Department of Mathematics \\
University of California, Davis\\
Davis, CA 95616, USA}
\author[2]{Scott Mahan}
\affil[2]{Department of Mathematics\\
University of California San Diego \\
La Jolla, CA 92093, USA}
\author[3]{Timo Klock}
\affil[3]{Machine Intelligence Department\\
Simula Research Laboratory \\
Oslo, Norway}
\author[4]{Alexander Cloninger$^{2,}$}
\affil[4]{Halocio{\u g}lu Data Science Institute\\
University of California San Diego \\
La Jolla, CA 92093, USA}
\author[5]{Rongjie Lai}
\affil[5]{Department of Mathematics\\
Purdue University\\
West Lafayette, IN 47907, USA}

\date{}

\begin{document}


\maketitle

\begin{abstract}
Autoencoding is a popular method in representation learning. Conventional autoencoders employ symmetric encoding-decoding procedures and a simple Euclidean latent space to detect hidden low-dimensional structures in an unsupervised way. Some modern approaches to novel data generation such as generative adversarial networks askew this symmetry, but still employ a pair of massive networks--one to generate the image and another to judge the images quality based on priors learned from a training set. This work introduces a chart autoencoder with an asymmetric encoding-decoding process that can incorporate additional semi-supervised information such as class labels. Besides enhancing the capability for handling data with complicated topological and geometric structures, the proposed model can successfully differentiate nearby but disjoint manifolds and intersecting manifolds with only a small amount of supervision. Moreover, this model only requires a low-complexity encoding operation, such as a locally defined linear projection. We discuss the  approximation power of such networks and derive a bound that essentially depends on the intrinsic dimension of the data manifold rather than the dimension of ambient space. Next we incorporate bounds for the sampling rate of training data need to faithfully represent a given data manifold. We present numerical experiments that verify that the proposed model can effectively manage data with multi-class nearby but disjoint manifolds of different classes, overlapping manifolds, and manifolds with non-trivial topology. Finally, we conclude with some experiments on computer vision and molecular dynamics problems which showcase the efficacy of our methods on real-world data. 
\end{abstract}

\section{Introduction} 
It is commonly believed that real-world data concentrates near low-dimensional structures embedded in high-dimensional space \cite{Belkin2003, fefferman2016testing, mishne2019diffusion}. Therefore, detecting the hidden low-dimensional structure is a crucial problem in data science, often referred to as manifold learning. 
Auto-encoding~\cite{bourlard1988auto,hinton1994autoencoders,liou2014autoencoder} is a key concept to learn low-dimensional structures from high-dimensional observations via dimensional reduction for efficient data representation in an unsupervised manner. It enables wide applications in many important tasks such as dimensional reduction \cite{bengio2015deep}, novel data generation \cite{kingma2013auto}, anomaly detection \cite{an2015variational}, molecular property prediction \cite{lim2018molecular}, dynamic system prediction \cite{otto2019linearly} and many more \cite{ng2011sparse, bengio2013representation, xu2018unsupervised, kusner2017grammar, donati2020deep}.

Conventional autoencoders (and their variants) map data to a low-dimensional Euclidean domain. However, the desired continuity and invertibility of the encoder and decoder indicate that the flat geometry of the latent space is often too simple to appropriately represent a data manifold with non-trivial topology. For example, it is well-known that the unit sphere in $\mathbb{R}^3$ is not homeomorphic to $\mathbb{R}^2$--not to mention a manifold with holes, a manifold with multiple disconnected components, or intersecting manifolds. Therefore, decoded samples drawn from the simple euclidean latent space will be unlike the training distribution if the underlying distribution is concentrated on a low-dimension manifold with disconnected components, holes or overlaps as shown in  \cite{mishne2019diffusion, schonsheck2019chart, tanielian2020learning}. A common method to overcome this problem is to map the data to a higher dimensional latent space. This approach presents two inescapable problems, first, the increased dimension of the latent space makes analysis much more difficult and second the latent distributions are not guaranteed to be compact, further increasing the diffusely of latent exploration. 


Another important question is how to create interpretable embeddings given additionally supervised (or semi-supervised) information. For example, labeling can help differentiate nearby but disjoint manifolds of different classes. Consider simultaneously learning a dataset of handwritten digits and the labels of these digits as a function defined on the data. This distribution has intersections (or at least near intersections) if there are poorly written digits (ex: 6s and 0s that look alike). Conventional Euclidean latent representations of these intersections insist that the data be in one of these classes. However, it is more reasonable to model these points as intersection of two manifolds and treat these points as being members of each class. Similarly, continuous labeling can also be helpful in structuring a latent space so that it is amenable for downstream analysis. Consider a dataset of pictures taken of a single scene from multiple angles. The camera angle of these pictures can be used to organize the pictures so that one direction in the latent space represents a change in the vertical angle and an orthognal direction representes a change in the horizontal. 

In this work, we present a method for simultaneously learning manifolds and labels given semi-supervised or fully-supervised information. To overcome data manifolds with complicated topology and possible nearby intersecting components, we learn the data manifold via a chart-atlas construction inspired by the chart autoencoder introduced in~\cite{schonsheck2019chart}. In addition, we present a  mechanism to encode overlapping manifolds with different labels by incorporating semi-supervised label information. We extend this framework to learn data manifolds and functions thereon simultaneously in supervised and semi-supervised regimes. 

Unlike conventional auto-encoders whose encoding and decoding steps share similar structures and complexity, the proposed encoding process enjoys a very low computational complexity. This asymmetric encoding-decoding procedure significantly reduces the computation cost. Theoretically,  our approximation analysis can be conducted with neural networks of width and depth that scale with the latent dimension of the data and only a weak dependence on the ambient space dimension. Finally, we consider numerical experiments on both synthetic and real-world data to validate the effectiveness of the proposed method. 

To summarize,  our main contributions are:
\begin{enumerate}\item We introduce a mechanism to incorporate semi-supervised information (labels) into the chart predictor/encoder. This allows for encoding overlapping manifolds with different labels;
\item We explore asymmetric encoder-decoder auto-regressive modes, which save computation cost. Specifically we examine models with linear encoders patched together though a partition of unity function based on radial basis networks; 
\item We theoretically discuss the approximation power of such networks that minimizes the influence of the ambient dimension; in particular we improve upon the theoretical approximation bounds of \cite{liao2019adaptive, shaham2018provable, schonsheck2019chart, aizenbud2021regression} for neural manifold regression

\end{enumerate}

\paragraph{Related Work}
Extracting the low-dimensional structures by modeling data as being sampled from some unknown manifolds has led to many dimension reduction techniques in the past two decades~\cite{tenenbaum2000global, roweis2000nonlinear, cox2001, Belkin2003, He2003, zhang2004principal, Kokiopoulou2007, maaten2008visualizing, liao2019adaptive}. Auto-encoding and variational auto-encoding~\cite{bourlard1988auto,hinton1994autoencoders,liou2014autoencoder,kingma2013auto} are important techniques to learn low-dimensional structures in an unsupervised manner. 
These models' unsupervised and oversimplified latent spaces make them insufficient to handle data with nontrivial topology, disconnected components, or multiple nearby/intersecting components. 

Recently, there has been an increasing interest in introducing non-Euclidean latent space for better representation.  \cite{falorsi2018explorations} introduced a non-Euclidean latent space to guarantee the existence of a homeomorphic representation, realized by a homeomorphic variational autoencoder. However, this approach requires that the topology of the dataset is known. Similarly, several recent works~\cite{davidson2018hyperspherical, rey2019disentanglement, connor2019representing} have introduced autoencoders with (hyper-)spherical or other, fixed non-euclidean latent spaces. Without assuming prior knowledge of the topological structure of the data, chart autoencoders (CAE) \cite{schonsheck2019chart} allow for the homeomorphic representations of distributions with arbitrary topology. The authors also show that there exists a CAE which can faithfully represent any manifold but that the number of charts may depend on the covering number of the manifold. In this work, we improve their worst-case scenario bound and prove that for many cases the number of charts needed is much lower. We further extend CAE to incorporate with semi-supervised information for detecting manifolds with overlapping structures. A similar class of chart-based models was recently pros posed by \cite{floryan2022data}. These methods decouple the chart prediction and econding problems by first dividing the training set into overlapping clusters via a K-means neighbors algorithm, then learns an seperate encoder for each of the clusters. In this paper, we refer to this as a divide and conquer approach. Our method has two main advantages over this approach. First, we parameterize the chart prediction function with a neural network that allows us to capture datasets which are disconnected for which K-means based approaches fail. Second, since we learn the patches and encoder at the same time our patches can be of different sizes, allowing us to cover larger 'flat' regions with a single chart, and areas of higher curvature with a collection of smaller charts. Further discussion of this is presented in section \ref{sec:finite_distortion}.

It is well known that VAEs and generative adversarial networks (GANs) cannot accurately represent data from disconnected distributions \cite{tanielian2020learning, mishne2019diffusion}.
One crucial reason is that these models have connected latent space and continuous layers. Therefore, data observed from separable distributions become (computationally) inseparable in the latent space.
For example, given two points $x$ and $y$ taken from disconnected manifolds, a non-charted autoencoder generates latent representations $z_x$ and $z_y \in \RR^n$ where $n$ is the dimension of the latent space. These points can be connected with continuous curves and since these layers in the models are continuous, the outputs of the decoder (or generator in the GAN case) will also be continuous. Therefore the newly generated data will fall along a continuous path between the disconnected manifolds from which $x$ and $y$ are sampled. In some applications, this may be desirable, but for many it is not. For examples, interpolating between two head shots of two individuals may produce an array of novel faces which have not been seen before and could be usefully for synthetic data generation or artistic pursuits. However, it is unlikely that many of the images produced by interpolating between "2" and "5" in a network trained on MNIST handwritten digits have any value. 

Several methods \cite{azadi2018discriminator,pmlr-v97-turner19a} have been proposed to overcome this problem during data generation via \textit{rejection sampling}, in which newly generated data is `thrown out' if it doesn't close enough to resembling something from the training data. Unfortunately, this results in `dead zones' in the latent space, which can not generate meaningful data. 
Similarly, other methods attempt to model the latent space using manifold learning techniques \cite{radhakrishnan2018memorization, mishne2019diffusion, li2020variational} on the latent space. However, these require more complex sampling schemes and essentially employ autoencoding as a black-box nonlinear dimensional reduction technique. Our approach completely avoids this problem by employing compactly supported charts, which are organized via a partition of unity. As a result, we provide a more interpretable latent space and do not require any post-processing in data generation.  

Many provable neural network approximation results~\cite{csaji2001approximation, gyorfi2006distribution, shaham2018provable, chen2019efficient, cloninger2021deep} for low dimensional manifolds focus on approximating a function $f$ supported on or near some smooth $d$-dimensional manifold isometrically embedded in $\RR^D$ with $d \ll D$. However, these results often assume the manifold is both known and well behaved. 
Recently, several efforts have explicitly bounded the size of networks needed to represent an unknown manifold \cite{schonsheck2019chart} or estimate it with non-parametric methods \cite{aizenbud2021non}. This work  
improves upon these results to characterize the size of networks needed to accurately approximate manifolds under very general regularity conditions on the chart maps. We show the approximation power of such networks that essentially depends on the intrinsic dimension of the data manifold. The number of charts needed to cover many manifolds is surprisingly few as shown later in Section \ref{sec:finite_distortion}.  


\section{Description of the Problem}
We begin by describing the mathematical model of data which motivates the proposed representation. The \textit{fuzzy manifold} hypothesis states that coherent data in high dimension concentrates \textit{near} lower-dimensional manifold(s) \cite{Belkin2003, mishne2019diffusion}. For example, a dataset of $256 \times 256$ natural images of $k$ different objects rotating along a single axis in 3D will form a set of $k$ 1D manifolds in $\RR^{256^2}$. In addition, we would like to learn a function on this data. Here, this function may be categorical (ex:labels of the picture) or continuous (ex:the rotation angle). 

Different from conventional regression problems on a manifold where the manifold structure is given, we assume that the ground manifold structure of the data is unknown. A naive way to jointly learn the manifold and function is to concatenate the coordinate data and function values. Then one must learn the distribution in the concatenated ambient space. This problem is significantly harder due to the curse of dimensional and the fact that concatenating this label information can destroy the underlying structure of image manifold. For example, the image data wouldl likely be amenable to analysis with a convolutional neural network (CNN), but concatenating additional information to the image vector is likely to degrade the performance of this CNN. Other works have attempted to solve this problem by learning a structured latent space based on a topological priors \cite{davidson2018hyperspherical, connor2019representing, schonsheck2022spherical}. However, these approaches require problem-specific insights about the manifold's topology and the function's behavior, which are not available in many contexts. Alternatively, we propose a more computationally tractable strategy in which we factor the problem into a manifold learning problem and functional approximation problem.

\begin{figure}[h]
\centering
     \includegraphics[width=0.65\textwidth]{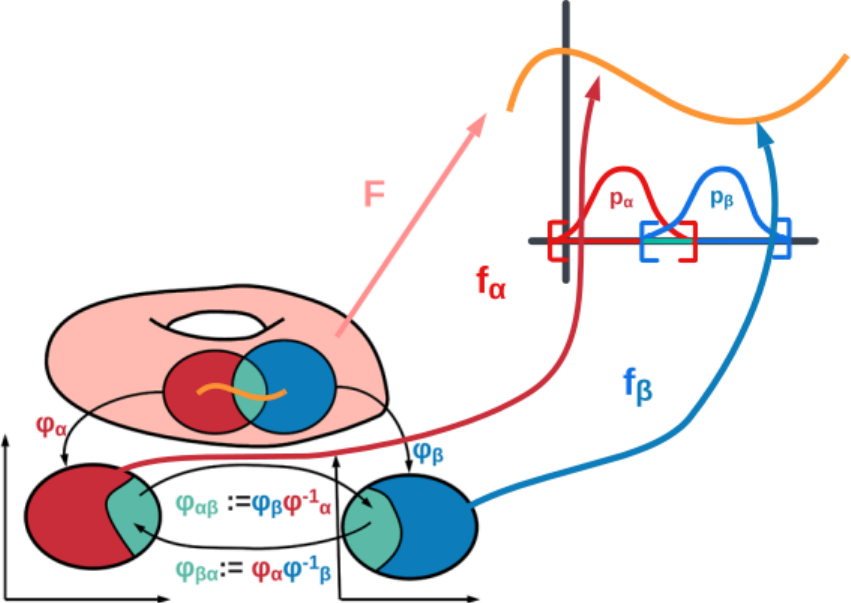}
     \caption{Represent a function $F:\M \rightarrow \mathbb{R}$ from samples of $F$ on a manifold $\M \subset \RR^D$ through charts. We first parameterize $\M$ with an atlas of overlapping charts $\{U_\alpha,\phi_\alpha\}_\alpha^N$, each of which is homeomorphic to $\mathbb{R}^d$. Then the function $F$ can be locally approximated a chart $U_\alpha$ by composing function $f_\alpha:\RR^d \rightarrow \RR$ with $\phi_\alpha$. That is, for $x \in U_\alpha$,  $F(x) \approx f_\alpha( \phi_\alpha (x))$}
     \label{fig:philosophy}
\end{figure}

More precisely, we consider the case when data is sampled on a union of finite $d$-dimensional manifolds $\M = \cup_i \M_i$ in an ambient space  $\mathbb{R}^D$. 
Based on a dataset sampled on $\M$ and samples of a function $F:\M \rightarrow \mathbb{X}$, we would like to simultaneously learn the data manifold and the function thereon. We remark that $F$ can be categorical (namely, vector valued such as a one hot encoding in which case $\mathbb{X}$ would be a probability space over the number of classes) in a classification problem or continuous in a regression problem in which case $\mathbb{X}$ would be $\RR^n$ for some $n$. In particular, samples of values of $F$ are viewed as semi-supervised information which allow us to impose a structure on the latent space.

Our overall approach is demonstrated in Figure \ref{fig:philosophy}, whereby we learn a collection of locally supported overlapping charts $\{U_\alpha,\phi_\alpha\}_\alpha$ to parameterize the manifold and then approximate the desired function $F$ on the manifold by gluing together its local approximations via a partition of unity functions $\{\rho_\alpha\}_\alpha$. Recall that the partition of unity functions satisfy: a) $\text{Suppt}(\rho_\alpha) \subset U_\alpha$; b) for any point $x$, there is a neighbourhood of 
where only a finite number of the functions of $\rho_\alpha$ are not 0 and $\sum_{\alpha} \rho_\alpha(x) = 1$.  
Figure \ref{fig:arch} details how we realize this mathematical model with learnable network modules.

\paragraph{Topological Concerns}

Chart autoencoders (CAEs) were introduced in \cite{schonsheck2019chart} to represent data from manifolds with non-trivial topology without prior knowledge of the topology. The authors introduce the concept of an \emph{$\epsilon$-faithful representation}, which quantitatively measures the topological and geometric approximation of autoencoders to the data manifold. An essential result of this work is that an auto-regressive model cannot accurately represent a manifold distributed dataset unless the latent space is homeomorphic to the data manifold. This motivates the use of multi-latent space models in which each chart covers only a compact region.

Consider a double torus shown in Figure~\ref{fig:Eight}. In the case where the latent dimension is two, they fail to generate data to cover the whole manifold, in this case the manfiold has been 'flattened'. When the latent dimension is increased to three, the newly sampled data do not all stay on the manifold. Additionally, near by points on the manifold may be far apart in the latent space because homeomorphism is violated. In either case, we can see that even if the reconstruction of the training data is good, sampling from from the latent space does not produce a novel data-set which mimics the toplogical propterites

\begin{figure}[ht]
  \centering
  \includegraphics[width=.8\linewidth]{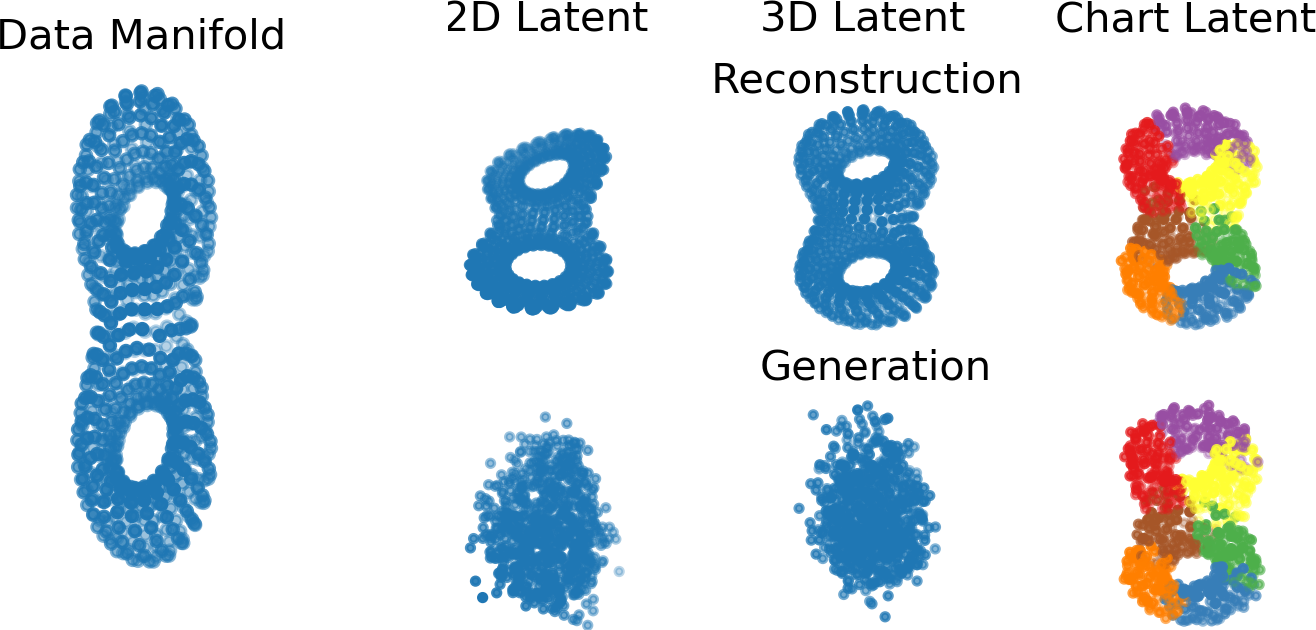}
  \caption{\textbf{Left:} Training Data \textbf{Right:} Reconstruction of training data and sampling of new data using a 2D VAE, a 3D VAE and a 2D Chart autoencoder}
  \label{fig:Eight}
\end{figure}

Instead, on the rightmost column of Figure~\ref{fig:Eight} the same double torus is parameterized by using seven color-coded charts. Here we observe accuracy in both Reconstruction and generation This example exhibits several characteristics and benefits of the proposed work: (i) the charts collectively cover the manifold and faithfully preserve the topology (two holes); (ii) the charts overlap (as evident from the coloring); (iii) new points sampled from the latent space remain on the manifold; (iv) nearby points on the manifold remain close in the latent space; and (v) because of the preservation of geometry, one may accurately estimate geometric proprieties (such the geodesics and curvature).

\paragraph{Approximation Bounds}

By explicitly characterizing trade-offs between the number of charts in the representaion, complexity of the sub-networks, local geometry of the data manifold manifold and sampling rate of the training data, we improve upon previous theoretical approximation bounds of \cite{liao2019adaptive, shaham2018provable, schonsheck2019chart, aizenbud2021regression} for manifold learning. In particular, we prove that the number of latent space charts can actually be quite small and does not need to scale with the covering number of the manifold. In some instances, the number of charts is independent of ambient dimension and the number of data points altogether. Instead, we establish a finite distortion condition that only requires the encoder to be injective. Another significant benefit of this relaxation is that it allows the encoders to be quite simple, potentially even linear projections of large pieces of the manifold. Moreover, this trade-off is shown not to affect the decoder complexity significantly.

\paragraph{Incorporating Supervision}
As an extension of the CAE architecture, we propose to learn manifold structure and functions thereon simultaneously.  
Once we have learned the local chart structure of the data manifold, we organize them via a partition of unity function, making learning $F$ much easier. Different from conventional autoencoders which do not combine with supervised information, this additional information can be used to improve the manifold approximation in a supervised (or semi-supervised) manner. This is an important step to handle data with multi-class nearby but disjoint manifolds. If the function to be learned is categorical, then we use constant (but learnable) functions $f_\alpha$ to separate the classes to learn disconnected or overlapping manifolds. By enforcing that each chart covers examples from one class of the data, we can trivially factor the problem into single-class subproblems. Then, after the model is trained, we can generate novel data from any category by sampling from the appropriate chart latent spaces.


\begin{figure}
\centering
\minipage{0.24\linewidth}
\centering
  \includegraphics[width=.99\linewidth]{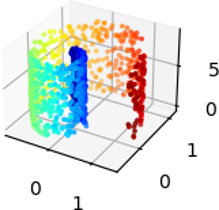}\\
  \centering (a)
\endminipage\hfill
\minipage{0.24\linewidth}
\centering
  \includegraphics[width=.99\linewidth]{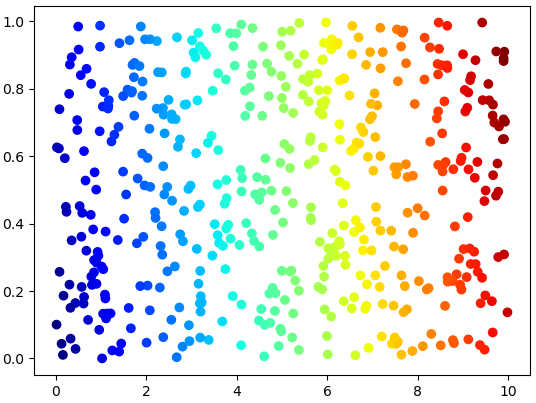}\\
  \vspace{0.3cm}
    \centering (b)
\endminipage\hfill
\minipage{0.24\linewidth}
\centering
  \includegraphics[width=.99\linewidth]{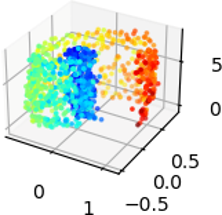}\\
    \centering (c)
\endminipage\hfill
\minipage{0.24\linewidth}%
\centering
  \includegraphics[width=.99\linewidth]{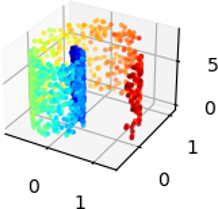}\\
    \centering (d)
\endminipage
\caption{A function on the Swiss roll data. (a) Training Data; (b) Function on flattened data; (c) Points generated by 4D autoencoder; (d) Points generated by learning latent representation and function thereon.}
\label{fig:Swiss1}
\end{figure}

One advantage of this approach is it allows for intrinsic learning of functions on the manifold other than categorical labels. Since each local approximation maps to a Euclidean latent space, we can use standard functional approximation techniques such as constant, linear, polynomial, or multilevel perceptrons to approximate the given function. Finally, by composing this approximation with the chart functions, we approximate $F$. Figure \ref{fig:Swiss1} shows a Swiss roll example of this where the ambient dimension is 3, the intrinsic dimension is 2, $F$ maps to $\RR$ and is linear along certain geodesics of the manifold. It is evident that the function to be learned is linear with respect to the intrinsic metric of the manifold. For this case, learning the function with knowledge of the intrinsic metric only requires four more parameters (i.e. one linear layer). In contrast, one naive approach to learning a representation of this data would be to concatenate the coordinates of the Swiss roll with the function data and then train an autoencoder on $\RR^4$. This requires many more parameters since the dimension of the problem is greater. Another disadvantage of this naive approach occurs when the information of $F$ is not complete--i.e. when there are points on the manifold which are unlabeled. Standard models cannot use these to train the concatenated representation. Our approach involves learning a manifold representation $\{U_\alpha, \phi_\alpha\}$ and a latent function approximation $f:\RR^2 \rightarrow \RR$, then approximating the $F \approx f_\alpha( \phi_\alpha)$ from the latent space. Here, we train a standard VAE on the four-dimensional data with roughly 5k parameters and a similarly sized manifold + latent function model with approximately 3k parameters. Despite being much smaller, the manifold + latent model does a better job of representing the roll's geometry as well as the values of the function.


\section{Description of Model}
\begin{figure}[h]
    \includegraphics[width=0.95\textwidth]{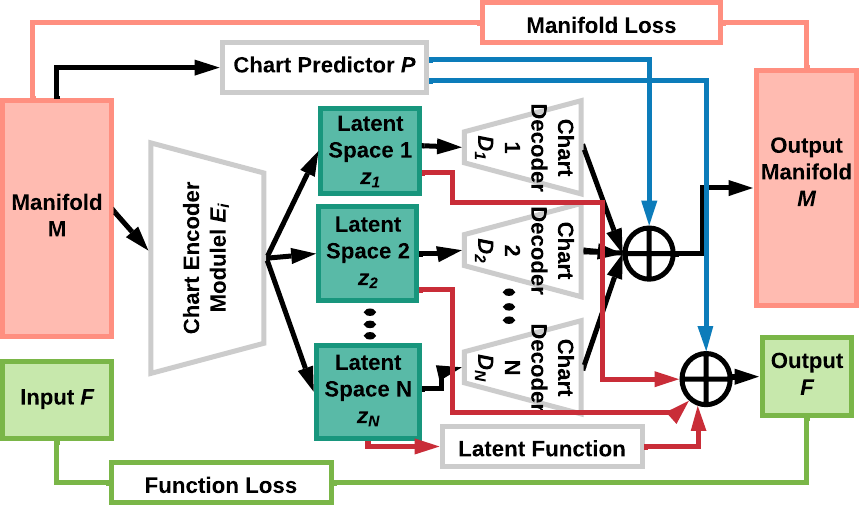}
    \caption{Network architecture. Black lines indicate the flow of manifold data, red indicates function approximation, blue represent chart probabilities. Boxes with gray outlines represent neural modules and colored boxes represent data spaces. }
    \label{fig:arch}
\label{fig:idea}
\end{figure}
The overall architecture of our model is presented in Figure \ref{fig:idea}, where the data manifold has been parameterized by a chart structure, and a function can be viewed as composition between chart parameterization and functions on the latent space.  
The encoder module $E$ is formed from a set of $N$ parallel, simple, independent encoders $E_i$.  Given an input point $x$, the encoder module, $E$, maps to a series of $N$ different latent spaces $Z_i$, each of which is a unit cube in the desired latent dimension. In section \ref{sec:finite_distortion}, we show that these sub-modules can be very simple--even linear projections. A set partition of unity functions, $\p =\{\p_i\}$ is used to predict which chart(s) contain the point, and therefore which latent representation should be used in the reconstruction. These functions, which we dub the chart predictor, defines a probability distribution on the data allowing us to ``glue'' the charts together in an overlapping or disconnected manner (see examples in Figure \ref{fig:trg}). Each latent space is equipped with a decoder module that maps back to the observation space, generating a reconstruction $y_i.$ We also define a latent function $f_i$, which can be combined with the encoder and partition of unity to approximate $F$ on each chart. That is $F(x) \approx \{f_i(E_i(x)) \big| i = \argmax_i \p_i (x)\}$. Given a trained model, to encode a single point we can use $\p(x)$ to determine which chart(s) to use for the encoding, saving us computation. Full descriptions of each of the sub-modules purposes and exact architectures for the numerical experiments are presented in the Supplementary Materials. 

\subsection{Description of Submodules}

\paragraph{Encoder and Decoder} We parameterize the chart functions $\{\phi_i \}$ and their inverses with multilevel perceptrons. The encoder module, $E$, maps input data to $N$ compact latent spaces. That is, $E:x \in \M \subset \RR^D \rightarrow \{\mu, \sigma\}_{i=1}^N \in \RR[0,1]^d \bigoplus \RR[0,1]^d $. Each of these latent spaces, which we sometimes refer to as  chart spaces, is equipped with its own decoder which maps from the latent space back to the data manifold $D_i: z_i \in M_i \subset \RR^D$. 

In traditional autoencoder frameworks, the encoder and decoder are often thought of as being approximate inverses of each other, and have symmetric architecture. However, this is unnecessary in our approach. Later, we show that non-linear data can be adequately modeled with a collection of local encoders formed by a single linear multiplication followed by a sigmoid activation function. The key insight is that as long as the encoder maps data from a local neighborhood with finite distortion, then local PCA provides a good enough approximation of the manifold that we can learn the an approximation exponential map (via the decoder). In general, this condition is not guaranteed to hold for traditional auto encoding frameworks. Empirically we observe that for may networks with more complicated encoding networks fail to meet this condition. Moreover, since each of our decoders only has to decode a small portion of the entire manifold, they can be much simpler than the global decoders more commonly used.

\paragraph{Latent Space} 
The so called `reparameteriztion trick' \cite{kingma2013auto} is used in variational autoencoders to ensure that the data is `nicely' distributed in the latent space. We adapt this idea to encourage the distributions within each chart to be concentrated near the center of the chart space. Then, rather than directly outputting a latent variable $z_i$ for the $i^{th}$ chart space, the encoder will instead predict a Gaussian distribution in the space parameterized by the output $\mu_i$, $z_i$. During training, we sample from these distributions and decode the samples, as is standard in variation auto-encoding frameworks.

\paragraph{Chart Predictor} Since the encoder model produces a latent representation of a given data point in each chart space we must choose which of the charts to use for reconstruction of said point. To do so we build a partition of unity over the data manifold. In the origal chart auto-encoder \cite{schonsheck2019chart} this was achieved though a neural network ending in a softmax layer: $P: x_i \rightarrow p \in \Omega$. More recently \cite{floryan2022data} employed a similar achetecure, but used a k-means-like algorithm to divide the data set into overlapping charts before training the encoder and decoder modules. This method guarantees that the predicted charts are compact (in the ambient space) and allow for each of the charts to be trained independently (or in parallel). However, this approach is not always the most efficient in terms of number of charts needed to cover a space. For example, \cite{floryan2022data} states that they need at-least 4 charts to cover a torus, when it can be homeomophically mapped to a single 2D latent space with period boundary conditions. 

In this work, we present a third option which hybridizes these approaches. We employ a varitant of a radial basis function network to determine the parition of untity between the charts. More specifically, we define the predictor as:

\begin{equation}
    P(x) = \mathbf{S} (y(x)^T w) \quad y_i(x) = \exp ( \frac{-||x-X_i ||^2}{ \sigma})
\end{equation}
where $w$ and $\sigma$ are learnable parameters, $\mathbf{S}$  is the softmax function, and the centeroids $X_i$ are determined by decoding the center of each latent space. The advantage of this approach is that we have a guarantee of local compactness similar to the k-means approach, but with additional flexibility to have anisotripicly parameterized charts which can be smaller on areas of the manifold witch high curvature and larger on more flat regions. 

\paragraph{Function Approximation} To approximate a function $F$ on the manifold we use a function defined on the chart space. Each chart is equipped with a $f_i: z_i \rightarrow \hat x_i \in \RR$. Depending on the application $f$ may be parameterized as a constant function, polynomial basis or neural network. If functional data is only available on some subset of the training data (i.e. the labels are incomplete) we only run this module on those data points. However, in this regime we can still use the unlabeled points to train the encoders and decoder modules

\subsection{Initialization} \label{app:initial}
In order to efficiently train local representations of the data we propose an initialization scheme to ensure that each of the chart functions initializes to different regions of the data. To train an $N$-chart model, given some training data $x\in \mathcal{M}$ with functional samples $f(x)$, we first perform farthest point k-means sampling on the lifted varifold $\mathcal{P} = \M \bigoplus f(\M)$ to obtain $N$ well distributed initialization points $\{\tilde{x}\}_{j=1}^N$. We then assign one of these points to each of the $N$ charts and train them to reconstruct the assigned point, with the latent representation at the center of the latent space. Note that since we represent the latent space of each chart as a unit hypercube $\RR^d [0,1]$, this center is the vector $\frac{\textbf{1}}{\textbf{2}}$. We also use these points to pre-train the chart prediction module. The initialization loss is:

\begin{equation}
\begin{split}\label{eqn:InitLoss}
  \mathcal{L}(x):=   \sum_{j=1}^N  \Big( |E(x)_j - \frac{\textbf{1}}{\textbf{2}}|_2^2  +  |D_j(\frac{\textbf{1}}{\textbf{2}})-\tilde{x}_j|_2^2 \\
  + |p_j - \delta_j|_2^2 +
  \mathcal{F}(f(x),f_j(\tilde{x_j})) \Big)
\end{split}
\end{equation}

With this in mind, we can think of training on the full dataset as a region growing scheme, whereby new points are added to the domain of each chart. However, we stress that during training the centers of these regions may move. Moreover, by employing the chart removal process detailed in \cite{schonsheck2019chart} we could further simplify the representation by eliminating charts that only cover a small portion of the data, or one for which the nearby charts are already doing an adequate representing. 

A natural question to ask is how many charts do we need to cover a given manifold. Theoretically, the minimum number is given by the  Lusternik-Schnirelmann category of the underlying space. However, finding this is computationally intractable. Instead we follow the philosophy of \cite{floryan2022data} and choose a number which we believe to be larger than necessary number. There is no harm in having too many charts, and it has been discussed in \cite{schonsheck2019chart} that unused charts can be automatically removed during the training process by monitoring the partition of unity functions.

\paragraph{Loss and Training}
During training, each chart decoder produces an output $y_i = \textbf{D}_i \circ \textbf{E}_i(x)$. If $x$ lies on the domain of only one chart, then we only require that this chart should properly reconstruct this point. If $x$ lies in an overlapping region, then all of the overlapping charts should minimize this error. To obtain a sensible partition of unity, $\{\p_{i}\}$, we measure cross-entropy between the prediction and the inverse of the reconstruction error and minimize it.  Similarly, the function approximation loss from each chart is also weighted by $\p$ to ensure that the functional loss is only propagated to the local charts. Then we have the following training loss:
\begin{align}
\begin{split}\label{eqn:TrainLoss}
  \mathcal{L}(x)&:=  \min_i |x - x_i|_2^2 
   + \lambda \sum_{i=1}^N \p_i \Big( |x-x_i|_2^2 \\  
   &+ \mathcal{F}(f(x),f_j(\tilde{x_j})) + 
   \mathcal{KL}(\mathcal{N}(\mu_i, \sigma_i), \mathcal{N}_{\frac{1}{2},\frac{1}{4}}  ) \Big)
\end{split}
\end{align}
where $\lambda$ is a hyper-parameter, $\mathcal{F}$ is cross-entropy loss for classification-type problems and mean square error for regression type problems, and $\mathcal{N}_{\frac{\textbf{1}}{\textbf{2}},\frac{1}{4}}$ is the $d$-dimensional normal distribution centered at midpoint of the unit cube , $(\frac{1}{2})^d$, with standard deviation $\frac{1}{4}$. Note that this is a different Gaussian than standard VAEs employ since our latent spaces are compact unlike the standard VAE setup.

\paragraph{Data Generation}
After training, the model can generate novel data by sampling the latent space and decoding the samples. To create a dataset with a distribution close to that of the training set, we first measure how often each of the charts is used in the reconstruction of the training set. We then sample each of the charts proportionally to the number of times it was used in reconstruction. If the dataset has class labels, we can generate data from specific classes by sampling only from the appropriate charts. If the dataset is unlabeled, we can still generate data similar to a given an example by sampling from the latent space to which the given sample encodes and the neighboring charts (see Figure \ref{fig:trg}). This locality property is due to the compact support of charts and is not always the case for standard VAEs and GANs as shown by \cite{davidson2018hyperspherical, schonsheck2019chart}. 

\section{Theoretical Analysis}\label{sec:theory}

\paragraph{Low Complexity Encoder}  \label{sec:finite_distortion}
Rather than requiring a single latent space that is nearly Euclidean as with conventional autoencoders, we simply require a bi-Lipschitz mapping of each $d$-dimensional chart $U_i$. This allows us to reduce the number of charts needed, from a covering number that scales exponentially with the intrinsic dimension to a small number of charts independent of the intrinsic dimension in certain circumstances \cite{shaham2018provable, aizenbud2021regression}. This relaxed assumption also allows us to use simpler functions for the encoding operation. We seek a collection of bi-Lipschitz charts $(U_i,\phi_i)$ for our data manifold, which we call a finite distortion embedding. 

\begin{definition}
    An atlas $\{(U_i,\phi_i)\}_{i=1}^n$ is a \textbf{finite distortion embedding} of $\M$ if each chart function $\phi_i: U_i \to \real^d$ is bi-Lipschitz.
\end{definition}

The generality of this finite distortion condition also allows us to use a smaller number of charts than previous works. Later, we state and prove a somewhat restrictive but easily checked condition that guarantees a finite distortion embedding of $\M$ using only linear projections. We also conjecture an intuitive, more general sufficient condition for a finite distortion embedding that is more difficult to confirm for real-world data.

\paragraph{Chart Reach Condition}
We characterize the regularity of a given atlas parameterization by measuring the reach of each chart. Defined originally in \cite{federer1959curvature}, the reach of a set $X$ is the largest number $\tau_X$ such that every point within distance $\tau_X$ of $X$ has a unique nearest neighbor on $X$. More formally, given a closed set $X \subset \real^D$ and a point $z \in \real^D$, let $d(z,X) = \inf_{x \in X} \|z-x\|$ denote the distance from $z$ to $X$. The medial axis $\Med(X)$ of $X$ is define as the set of points in $\real^D$ that have at least two nearest neighbors in $X$:
\begin{equation*}
    \Med(X) = \{ z \in \real^D ~|~ \text{there exist } x \neq y \in X
    \text{ such that } \|x-z\|=\|y-z\|=d(z,X) \}.
\end{equation*}
The reach $\tau_X$ of $X$ is then defined as the distance from $X$ to $\Med(X)$:
\begin{equation*}
    \tau_X = d(X,\Med(X)) = \inf_{x \in X} d(x,\Med(X))  
    = \inf_{z \in \Med(X)} d(X,z).
\end{equation*}

The reach of a set captures the curvature and smoothness of the set and how close it is to self-intersecting. It can also be described as the radius of the largest ambient ball that can be freely ``rolled around'' the set. This intuition is illustrated in Figure \ref{fig:reach}.

\begin{figure}[H] 
\centering
\scalebox{1.35}{\begin{tikzpicture}
    \draw [black] plot[smooth,domain=-2:2] (\x, {(\x)^2/2});
    \draw [dashed, blue] (0,1) circle (1);
    \filldraw (0,1) circle (1pt);
    \draw [dashed] (0,1) -- (0,2.2);
    \draw [dashed, blue] (0,1) -- (0.965926,1-0.258819);
    \node (X) at (2.1,2) {\tiny $X$};
    \node (M) at (0.48,2.2) {\tiny $\Med(X)$};
    \node (T) [blue] at (0.45,0.75) {\tiny $\tau_\mathcal{M}$};
\end{tikzpicture}}
\caption{Geometric intuition of reach.}
\label{fig:reach}
\end{figure} 

Choosing charts with larger reach allows us to embed them via bi-Lipschitz linear projections, which together yield a finite distortion embedding of $\M$. Formally, we have the following result.

\begin{theorem} \label{thm:reach}
    Let $\delta>0$. Suppose each $U_i \subset \M \subset \real^D$ is a $d$-dimensional connected submanifold satisfying $\|u_1-u_2\| \leq 2\tau_{U_i} - \delta$ for all $u_1,u_2 \in U_i$. Then for each $U_i$ there is a $d$-dimensional subspace $\H_i \subset \real^D$ such that the projection map $\pi_{\H_i} : U_i \to \H_i$ is bi-Lipschitz with lower Lipschitz constant $\big( 1 + (D-d)\tfrac{2\tau_{U_i}}{\delta}\big)^{-1/2}$.
\end{theorem}

\begin{proof}
    See Appendix \ref{app:reach}.
\end{proof}

Theorem \ref{thm:reach} gives a sufficient condition for $\{(U_i,\pi_{\H_i})\}_{i=1}^n$ to be a finite distortion embedding of $\M$. The reach of each chart has an intuitive geometric meaning and can often be calculated exactly for simple manifolds or approximated computationally for more complex ones \cite{aamari2019estimating, suarez2021turing}. Below, we provide an example where the condition in Theorem \ref{thm:reach} describes the exact minimum number of charts needed to encode each one injectively via a simple linear projection.

\paragraph{Example 1} Let $\M = \S^1 \subset \real^2$. In polar coordinates, $\S^1 = \{ s=(r,\theta) \,|\, r=1,~ \theta \in [-\pi,\pi]\}$ where $-\pi$ and $\pi$ are identified. Then for $s \in \S^1$, define the charts
$U_1 = \{s \,|\, \theta \in (-\tfrac{\pi}{3}-\epsilon, \tfrac{\pi}{3}+\epsilon)\}$,
$U_2 = \{s \,|\, \theta \in (\tfrac{\pi}{3}-\epsilon, \pi+\epsilon)\},$ and 
$U_3 = \{s \,|\, \theta \in (-\pi-\epsilon, -\tfrac{\pi}{3}+\epsilon)\}$ 
for $\epsilon>0$. Consider $s_1 = (1,\tfrac{\pi}{3})$ and $s_2 = (1,-\tfrac{\pi}{3})$ in $\S^1$ as in the left of Figure~\ref{fig:normal_vectors}. Then $\|s_1-s_2\| = 2\sin(\pi/3) = {\sqrt{3}}$ in $\real^2$. While $s_1$ and $s_2$ are not the farthest apart points in $U_1$, given any $\eta>0$ we can choose $\epsilon$ small enough so that $\|u_1-u_2\| < \sqrt{3}+\eta$ for all $u_1, u_2 \in U_1$. The reach of each chart $U_i$ is $\tau_{U_i}=1$.
Thus, each $U_i$ satisfies the assumption of Theorem \ref{thm:reach} with $\delta=0.25$ for $\epsilon$ small enough. 
The chart $U_1$ can be linearly projected onto the space $\H_1 = \{(0,y) \in \real^2 \,|\, y \in \real\}$ via the map $\pi_{\H_1}(x,y) = y$, as shown in Figure \ref{fig:normal_vectors}. This projection is clearly injective and bi-Lipschitz, and similar maps exist for $U_2$ and $U_3$. Thus, this choice of charts admits a finite distortion embedding of $\M$.

Note that in this example, we can find a finite distortion embedding via linear projections if and only if each chart has an arc length less than $\pi$. As soon as the arc length of a chart $U_i$ exceeds $\pi$, we can choose two antipodal points $u_1, u_2 \in U_i$ such that $\|u_1-u_2\| = 2\tau_{U_i} = 2$, so the condition of Theorem \ref{thm:reach} is not satisfied. Moreover, $U_i$ cannot be projected injectively onto a linear subspace in this case since it includes a semicircle. Hence, in this example, the condition of Theorem \ref{thm:reach} coincides exactly with the existence of a finite distortion embedding via linear projections.

\paragraph{Normal Vector Characterization}
Theorem \ref{thm:reach} only provides a sufficient condition for a manifold to have a finite distortion embedding. In reality, there is a much larger class of manifolds, both smooth and non-smooth, that have finite distortion embeddings. Moreover, we can often reduce the number of charts below the number required by Theorem \ref{thm:reach} and still obtain a finite distortion embedding.

\paragraph{Example 2} 
Let $\M \subset \real^2$ be a triangle with sides $T_1$, $T_2$, and $T_3$, parameterized as
$T_1 = \{(x,0) \,|\, x \in [-1,1]\}$, 
$T_2 = \{(x,x+1) \,|\, x \in [-1,0]\}$, and
$T_3 = \{(x,1-x) \,|\, x \in [0,1]\}$.

If we allow closed charts, we can choose $U_i=T_i$ so that each chart has infinite reach (being just a line segment). These charts satisfy Theorem \ref{thm:reach} and each side of the triangle is already contained in an affine subspace of $\real^2$. However, if we ignore the requirements of Theorem \ref{thm:reach} we can get a finite distortion embedding of $\M$ with only 2 charts. In particular, set $U_1=T_1$ and $U_2 = T_2 \cup T_3$ with the projection map $\pi_{\H_2}(x,y) = x$, as shown in the center of Figure~\ref{fig:normal_vectors}. These charts clearly admit a finite distortion embedding, but they do not satisfy the assumptions of Theorem \ref{thm:reach} since $\tau_{U_2}=0$.

\begin{figure}[h]
    \centering
    \def\r{2.4}
    \scalebox{0.7}{\begin{tikzpicture}
        \draw [dashed] (0,0) circle (\r);
        \draw [ultra thick, blue] (1.2,-2.078461) arc (-60:60:\r);
        \draw [-stealth, blue] (\r,0) -- (\r+0.5,0);
        \draw [-stealth, blue] (2.318222,0.621166) -- (2.801185,0.750575);
        \draw [-stealth, blue] (2.318222,-0.621166) -- (2.801185,-0.750575);
        \draw [-stealth, blue] (2.078461,1.2) -- (2.511474,1.45);
        \draw [-stealth, blue] (2.078461,-1.2) -- (2.511474,-1.45);
        \draw [-stealth, blue] (1.697056,1.697056) -- (2.050610,2.050610);
        \draw [-stealth, blue] (1.697056,-1.697056) -- (2.050610,-2.050610);
        \node (U) [blue] at (2.12,1.62) {$U_1$};
        \draw [stealth-stealth, very thick, red] (0,-2.8) -- (0,2.8);
        \node (H) [red] at (0.4,2.6) {$\mathcal{H}_1$};
        \node [black, circle, fill, inner sep=1.5pt, label={[label distance=0cm]below:$s_1$}] at (1.2,2.078461) {};
        \node [black, circle, fill, inner sep=1.5pt, label={[label distance=0cm]below:$s_2$}] at (1.2,-2.078461) {};
    \end{tikzpicture}}
\hfill
    \def\l{2.4}
    \scalebox{0.8}{\begin{tikzpicture}
        \draw [stealth-stealth, very thick, red] (-2.65,0) -- (2.65,0);
        \draw [ultra thick, blue] (-\l,0) -- (0,4.156922);
        \draw [ultra thick, blue] (0,4.156922) -- (\l,0);
        \draw [dashed] (-\l,0) -- (\l,0);
        \node (U) [blue] at (0.45,4.0) {$U_2$};
        \node (H) [red] at (2.45,-0.3) {$\mathcal{H}_2$};
        \draw [-stealth, blue] (-\l+0.6,1.039230) -- (-\l+0.6-0.433013,1.039230+0.25);
        \draw [-stealth, blue] (\l-0.6,1.039230) -- (\l-0.6+0.433013,1.039230+0.25);
        \draw [-stealth, blue] (-\l+1.2,2.078461) -- (-\l+1.2-0.433013,2.078461+0.25);
        \draw [-stealth, blue] (\l-1.2,2.078461) -- (\l-1.2+0.433013,2.078461+0.25);
        \draw [-stealth, blue] (-\l+1.8,3.117691) -- (-\l+1.8-0.433013,3.117691+0.25);
        \draw [-stealth, blue] (\l-1.8,3.117691) -- (\l-1.8+0.433013,3.117691+0.25);
        \draw [-stealth, blue, dashed] (-\l+2.4,4.106922) -- (-\l+2.4-0.433013,4.106922+0.25);
        \draw [-stealth, blue, dashed] (\l-2.4,4.106922) -- (\l-2.4+0.433013,4.106922+0.25);
        \draw [-stealth, blue, dashed] (0,4.106922) -- (0,4.6);
    \end{tikzpicture}}
\hfill
    \scalebox{1.0}{\begin{tikzpicture}
        \draw [stealth-stealth, very thick, red] (-5.3,-1.5) -- (0,-1.5);
        \node (H) [red] at (-0.2,-1.8) {$\mathcal{H}$};
        \node (U) [blue] at (-0.2,-0.25) {$\mathcal{M}$};
        
        \draw [ultra thick, blue] (0,0) arc (30:150:1);
        \draw [ultra thick, blue] (-1.732051,0) arc (-30:-150:1);
        \draw [ultra thick, blue] (-3.464102,0) arc (30:150:1);
        
        \draw [-stealth, blue] (-0.866025+0.5,-0.5+0.866025) -- (-0.866025+0.5+0.25,-0.5+0.866025+0.433013);
        \draw [-stealth, blue] (-0.866025,-0.5+1) -- (-0.866025,-0.5+1+0.5);
        \draw [-stealth, blue] (-0.866025-0.5,-0.5+0.866025) -- (-0.866025-0.5-0.25,-0.5+0.866025+0.433013);
        \draw [-stealth, blue] (-0.866025-0.866025,-0.5+0.5) -- (-0.866025-0.866025-0.433013,-0.5+0.5+0.25);
        
        \draw [-stealth, blue] (-2.598076+0.5,0.5-0.866025) -- (-2.598076+0.5-0.25,0.5-0.866025+0.433013);
        \draw [-stealth, blue] (-2.598076,0.5-1) -- (-2.598076,0.5-1+0.5);
        \draw [-stealth, blue] (-2.598076-0.5,0.5-0.866025) -- (-2.598076-0.5+0.25,0.5-0.866025+0.433013);
        
        \draw [-stealth, blue] (-4.330127+0.5,-0.5+0.866025) -- (-4.330127+0.5+0.25,-0.5+0.866025+0.433013);
        \draw [-stealth, blue] (-4.330127+0.866025,-0.5+0.5) -- (-4.330127+0.866025+0.433013,-0.5+0.5+0.25);
        \draw [-stealth, blue] (-4.330127,-0.5+1) -- (-4.330127,-0.5+1+0.5);
        \draw [-stealth, blue] (-4.330127-0.5,-0.5+0.866025) -- (-4.330127-0.5-0.25,-0.5+0.866025+0.433013);
        
        \node [black, circle, fill, inner sep=1.5pt, label={[label distance=0cm]left:$s_1$}] at (-5.19615242271,0) {};
        \node [black, circle, fill, inner sep=1.5pt, label={[label distance=0cm]right:$s_2$}] at (0,0) {};
    \end{tikzpicture}}
\caption{Example choices of charts that admit a finite distortion embedding of a 1-dimensional manifold $\M \subset \real^2$. In each case, the normal vectors (blue) of each chart all lie in the same half space.}
\label{fig:normal_vectors}
\end{figure}

Motivated by Example 2, we prove a more general condition for a finite distortion embedding that places requirements on the Gauss map of each chart. The Gauss map gives a mapping from every point on a surface to its corresponding normal vector, which generates an embedding onto the unit sphere. Specifically, the Gauss map $N: X \to \mathcal{S}^{D-1}$ of a surface $X \subset \RR^D$ is a continuous map such that $N(x)$ is a unit vector normal to $X$ at $x$. The map is well-defined for a smooth orientable surface. Our next result characterizes a finite distortion embedding by the existence of a half space that contains the image of the Gauss map $N_i$ for each chart.

\begin{theorem} \label{thm:gauss}
    Let $\mathcal{M}$ be a smooth orientable $(D-1)$-dimensional manifold in $\RR^D$ satisfying $d_\mathcal{M}(u_1,u_2)^2 \leq C\|u_1-u_2\|^2$ (with $C>1$) for all $u_1,u_2 \in \mathcal{M}$, where $d_\mathcal{M}$ is the geodesic distance on $\mathcal{M}$. Suppose each $U_i \subset \mathcal{M}$ is geodesically convex and there exists a unit vector $n_i \in \R^D$ such that $\langle N_i(u), n_i \rangle \geq \delta \geq \sqrt{1-1/C}$ for all $u \in U_i$, where $N_i: U_i \to \mathcal{S}^{D-1}$ is the Gauss map of $U_i$. Then the linear projection $\pi_{\mathcal{H}_i} : U_i \to \mathcal{H}_i$ is bi-Lipschitz, where $\mathcal{H}_i = \{x \in \R^D \, | \, \langle x, n_i \rangle = 0\}$. 
\end{theorem}

\begin{proof}
    See Appendix \ref{app:gauss}.
\end{proof}

Example 2 satisfies the assumptions of Theorem \ref{thm:reach} with three charts, but not two as discussed. However, it satisfies the assumptions of Theorem \ref{thm:gauss} with only two charts, and does indeed admit a finite distortion embedding.

Moreover, the manifold on the right in Figure \ref{fig:normal_vectors} is an example that does not satisfy the conditions of Theorem \ref{thm:reach} with only one chart, since the distance between the labeled points $s_1$ and $s_2$ exceeds $2\tau_\mathcal{M}$. However, we can see that this manifold satisfies the assumptions of Theorem \ref{thm:gauss} and does indeed admit a finite distortion embedding via a single linear projection.

We conjecture that Theorem \ref{thm:gauss} holds under more general conditions. For one, the requirement that $d_\mathcal{M}(u_1,u_2)^2 \leq C\|u_1-u_2\|^2$ appears to be a proof mechanism, and we can likely relax the requirement to $\langle N_i(u), n_i \rangle \geq \delta > 0$. Moreover, we can generalize the Gauss map to non-smooth points by including a notion of `supnormal' vectors. A vector $v$ at a non-differentiable point $x_0$ is \textit{supnormal} if $v^\perp \cap N_\mathcal{M}(x_0) = \{x_0\}$ for a sufficiently small open neighborhood $N_\mathcal{M}(x_0)$ of $x_0$ on $\mathcal{M}$ (as illustrated by the the dashed arrows at the top of the triangle in Figure \ref{fig:normal_vectors}). We believe Theorem \ref{thm:gauss} holds provided that $\langle v, n_i \rangle \geq \delta>0$ for all supnormal vectors $v$ as well.

We also hope to generalize Theorem \ref{thm:gauss} to $d$-dimensional submanifolds of $\RR^D$ via a generalization of the Gauss map. These conjectures and more general conditions are a part of current and future work.

\paragraph{Decoder Complexity} \label{sec:decoder}

The complexity of auto-encoding a manifold generally depends on the topological and geometric properties of the manifold. By choosing charts with simple topology, we can use linear encoders and low-complexity decoders for this task. In particular, each decoder can be approximated by a ReLU network whose complexity depends mainly on the distortion of the decoder function and only weakly on the ambient dimension $D$.

\begin{theorem} \label{thm:decoder_complexity}
    Let $\{(U_i,\phi_i)\}_{i=1}^n$ be a finite distortion embedding of $\M$. Then for each decoder $D_i = \phi_i^{-1}$ and any $\epsilon \in (0,1)$, there is a ReLU network $\hat{D}_i : \real^d \to \real^D$ such that:
    \begin{enumerate}
        \item $\|D_i-\hat{D}_i\|_\infty < \epsilon$;
        \item $\hat{D}_i$ has at most $c\ln(1/\epsilon)$ layers and $c\epsilon^{-d}\ln(1/\epsilon)$ computation units followed by a single matrix multiplication of size $D \times c\epsilon^{-d}$, for some constant $c=c(d)$.
    \end{enumerate}
\end{theorem}

\begin{proof}
    See Appendix \ref{app:decoder_complexity}.
\end{proof}

Theorem \ref{thm:decoder_complexity} shows that we can approximate each decoder with a low-complexity ReLU network. In particular, all layers except for the final one are independent of the manifold data and its ambient dimension; they only serve to make a fine enough partition of the latent space. The final matrix multiplication consists of a finite number of decoded points to embed the data into $\real^D$. This network architecture is similar to a multidimensional version of the construction presented in \cite{yarotsky2017error}. 

In practice, we cannot directly implement the network $\hat{D}_i$ in Theorem \ref{thm:decoder_complexity} because we do not have access to the true decoder $D_i$, which determines the weights of the final layer. Our goal now is to approximate each decoder $D_i: [0,1]^d \to \real^D$ given noisy samples $\{(z_j,x_j)\}_{j=1}^n$. Given this sample, we approximate $D_i$ using a sum-product of a partition of unity with constant terms of vector-valued local polynomial regression functions. The result is an approximator that can be implemented by a ReLU network with weights dependent on known sample data, and whose convergence rate is independent of $D$. Combining the local polynomial regression convergence rates from \cite{aizenbud2021regression} with a fine enough partition of unity gives the following result, the proof of which is presented in Appendix \ref{app:decoder_statistical} along with the sampling assumptions on $\{(z_j,x_j)\}_{j=1}^n$.

\begin{theorem} \label{thm:decoder_statistical}
    Let $\{(U_i,\phi_i)\}_{i=1}^n$ be a finite distortion embedding of $\mathcal{M}$. Suppose each true decoder $D_i = \phi_i^{-1} \in C^k([0,1]^d)$ is $L_i$-Lipschitz. Under mild assumptions on a (possibly noisy) sample $\{(z_j,x_j)\}_{j=1}^n$, for any $\epsilon>0$ there are $C$ and $n$ large enough so that
    \begin{equation*}
        P\Big( \|D_i - \hat{D}_{i,n}\|_\infty > 2^d \cdot \big(Cn^{-\frac{k}{2k+d}} + L_i\sqrt{d} N^{-1}\big) \Big) \leq \epsilon,
    \end{equation*}
    where $N=\mathcal{O}\big(n^{\frac{1}{2k+d}}\big)$, 
    and $\hat{D}_{i,n}$ is computed via local polynomial regression at various grid points in the latent space.
\end{theorem}

\begin{proof}
    See Appendix \ref{app:decoder_statistical}.
\end{proof}

Given a large enough sample size $n$, Theorem \ref{thm:decoder_statistical} implies that we can approximate the decoder $D_i$ to a certain degree of accuracy with high probability, using just the sample data. We can then in turn approximate this local regression function with a ReLU network by approximating a partition of unity of the latent space, 
yielding a network that approximates $D_i$ and can be implemented in practice and trained via gradient descent.

\section{Numerical Experiments}\label{sec:numresults}
This section presents numerical experiments on synthetic and real-world data in supervised, semi-supervised, and unsupervised settings.  Additional ablation studies, along with complete model descriptions and training parameters are provided in the supplementary material.

\subsection{Unsupervised Geometric Examples.}
We begin with illustrative synthetic experiments that demonstrate the efficacy of learning manifolds. Figure \ref{fig:trg} illustrates learning two disconnected, non-smooth manifolds with non-trivial topology using a 5 chart model with very small sub networks. The data set is composed for 200 training points uniformly randomly selected from a pair of adjacent manifolds triangle. After training, each chart only covers part of one of the triangles, while their union covers the entire training set. Once the model is trained, we can cluster the points into two sets representing the two triangles forming a chart confusion matrix as shown in the left picture in Figure \ref{fig:trg}. To form this matrix, we uniformly sample the latent space of each chart, then measure the chart prediction activation of the other charts on this synthetic data. From the confusion matrix, we see that charts 1 and 2 overlap to cover one connected component and charts 3,4 and 5 cover the other component. This test demonstrates that our model can successfully handle manifolds with disconnected components and separate those components without the use of supervised labels.

\begin{figure}[h]
  \includegraphics[width=0.55\linewidth]{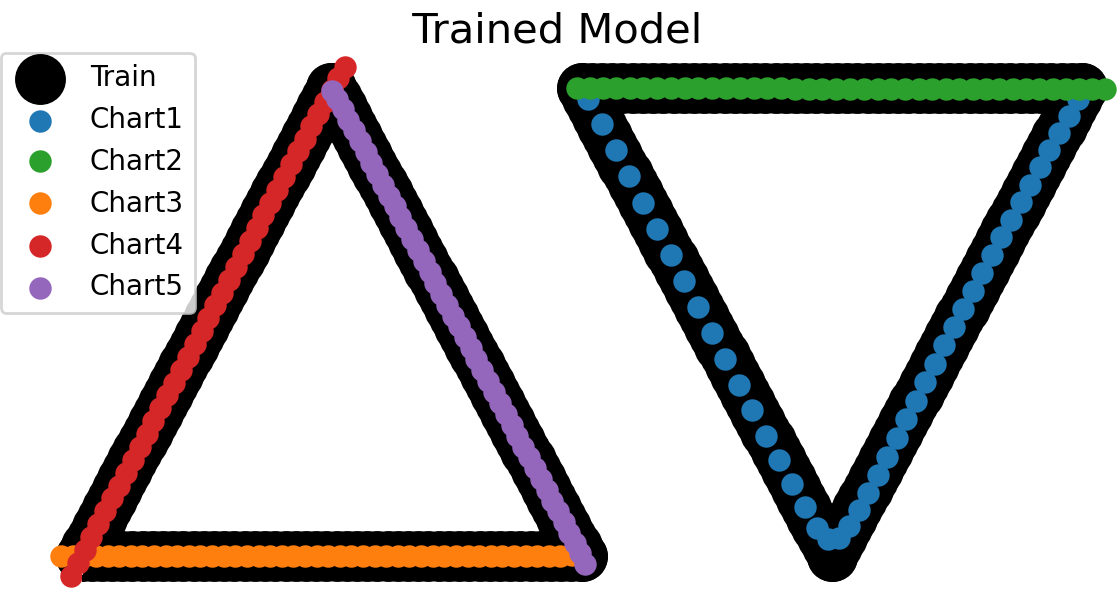}
  \hfill
  \includegraphics[width=0.36\linewidth]{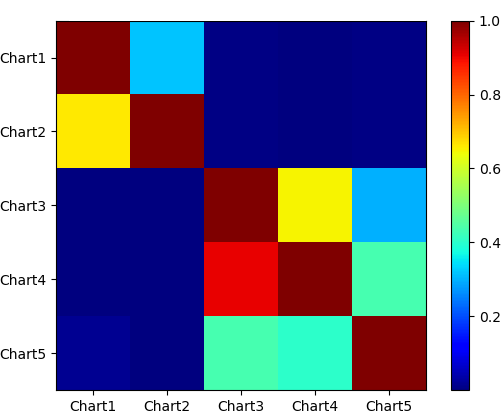}
  \caption{\textbf{Left}: Triangle training and reconstruction. \textbf{Right}: Chart overlaps via confusion clustering.}
  \label{fig:trg}
\end{figure}

\begin{figure}[h]
\centering
  \includegraphics[width=.95\linewidth]{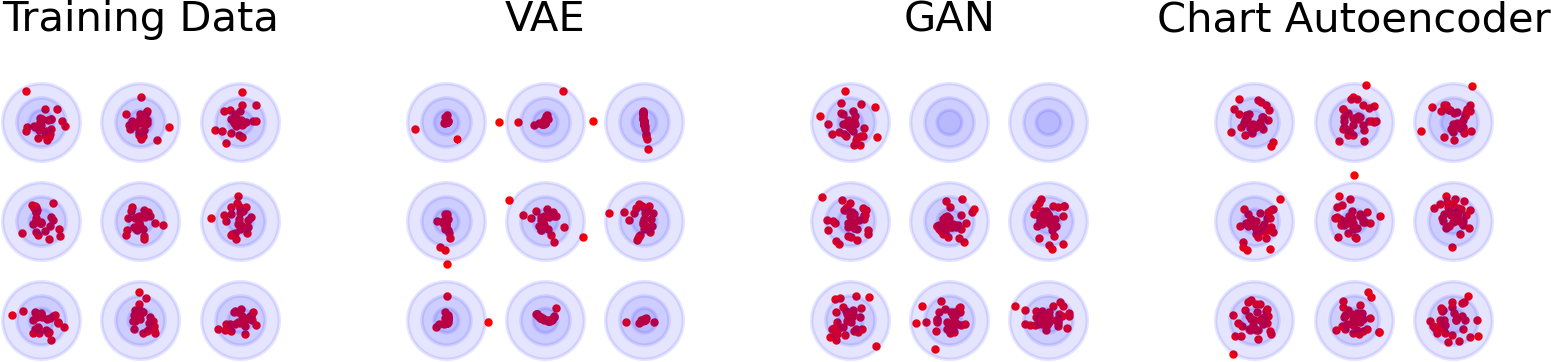}
  \caption{ \textbf{Left}: Sample of the training data,  \textbf{Second}: Sampling from trained VAE, note that the decoder connects the distributions since they are all covered by a model with continuous layers build on a continuous latent space. \textbf{Third}: Sampling form a GAN model, note the mode collapse. \textbf{Right}: CAE model which successfully covers and separates each distribution}
   \label{fig:gauss}   
\end{figure}

Next we demonstrate our methods resistance to the mode collapse phenomena while working in disconected domains. The example in Figure \ref{fig:gauss} shows a model with 2D latent spaces trained on data drawn from 9 disconnected Gaussian distributions. We initialize with 12 charts and use the chart removal procedure \cite{schonsheck2019chart} during training to reduce the number needed in the final representation to 9. We observe that the charts are disconnected and that the model can appropriately generate novel data by sampling the learned distribution. We emphasize that these samples are taken directly from the model and that there is no rejection sampling scheme in use. As a comparison, it is clear to see that conventional VAEs will generate data in between the disconnected components since their latent space is a connected Euclidean domain. The GAN model, does a better job at representing the disconnected nature of this dataset, but fails to cover each of the distributions. Conditional Autoencoders (\cite{van2016conditional}) cannot be applied to this directly to case since the points are unlabeled. Our method performs extremely well for generating data sampled from disconnected distributions.

\subsection{Supervised and Semi-Supervised Geometric Examples}

\begin{figure}[h]
\centering
      \includegraphics[width=.9\linewidth]{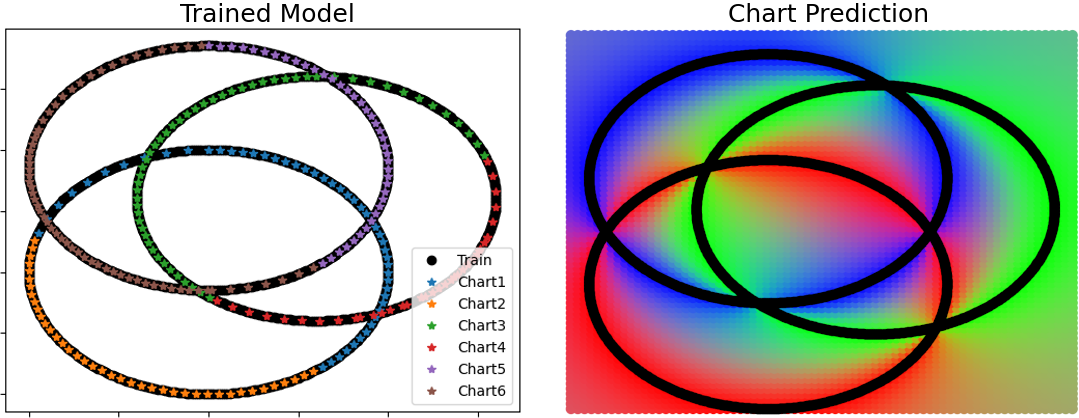}
      \caption{\textbf{Left} Overlapping circle training data and reconstruction. \textbf{Right} Chart prediction evaluated on the plane.}
      \label{fig:circles}
\end{figure}

\begin{table}
    \centering
    \resizebox{.95\columnwidth}{!}{%
\begin{tabular}{c|c|ccc}
\hline
Model                                                                                                             & \% Labeled          & Reconsturction            & Prediction             & Generation         \\ \hline
\multirow{2}{*}{\begin{tabular}[c]{@{}c@{}}Variational   Autoencoder \\      (latent dimension = 2)\end{tabular}} & \multirow{2}{*}{NA} & \multirow{2}{*}{3.12E-04} & \multirow{2}{*}{81.23} & \multirow{2}{*}{-} \\
                                                                                                                  &                     &                           &                        &                    \\ \hline
\multirow{3}{*}{\begin{tabular}[c]{@{}c@{}}Conditional   Autoencoder \\      (latent dimension = 2)\end{tabular}} & 10                  & 4.35E-03                  & 85.55                  & 85.45              \\
                                                                                                                  & 50                  & 7.36E-04                  & 94.29                  & 93.19              \\
                                                                                                                  & 100                 & 2.33E-04                  & 97.33                  & 96.23              \\ \hline
\multirow{3}{*}{\begin{tabular}[c]{@{}c@{}}Divide   and Conquer\\      (latent dimension  = 1)\end{tabular}}      & 10                  & 4.35E-04                  & 90.45                  & 89.35              \\
                                                                                                                  & 50                  & 7.73E-04                  & 93.29                  & 93.19              \\
                                                                                                                  & 100                 & 2.33E-04                  & 98.33                  & 97.23              \\ \hline
\multirow{3}{*}{\begin{tabular}[c]{@{}c@{}}Chart   Autoencoder\\      (latent dimension  = 1)\end{tabular}}       & 10                  & 2.13E-04                  & 94.13                  & 93.46              \\
                                                                                                                  & 50                  & 2.14E-04                  & 96.81                  & 95.28              \\
                                                                                                                  & 100                 & 2.13E-04                  & 99.32                  & 99.12              \\ \hline
\end{tabular}
    }
    \caption{Performance of models on overlapping circles dataset under varying levels of supervision}
    \label{table:circles}
\end{table}

Next, we show how functional information, such as class labels, can be used in conjunction with manifold learning to handle the case of overlapping manifolds. In Figure \ref{fig:circles}, the dataset is made of 500 points samples randomly from three mutually overlapping circles, each labelled with different class. By using a multi-chart model with a one-dimensional latent space and constant (but learnable) latent functions, we successfully separate the data into three classes, each covered by two charts. We compare these results with a tradition variational autoencoder and conditional autoencoder, both of which requires a 2-dimensional latent space (further tests with other latent spaces are considered in Supplementary \ref{apdx:additional}). Since the conditional autoencoder requires labels, we only train it on the labeled data in the abolition study presented in Table~\ref{table:circles}. Reconstruction measures the $l_2$-reconstruction error of points, prediction measures the classification accuracy (a task which the VAE and conditional autoencoder cannot conduct due to lack of function learning part) and finally generation measures the percent of generated points which belong to the desired class. It is clear to see that our model can produce very good results even with $10\%$ of labels.

\subsection{Coil-10 Dataset}

Next, we present results using a semi-supervised scheme on a real-world dataset. The Coil-10 dataset \cite{nene1996columbia, mishne2019diffusion} contains images of 10 physical objects rotating along one axis in 3D. This dataset can be modeled as 10 circles embedded in a high dimension ambient space. The difficulty in properly representing this dataset is that the minimum pixel distance between some of the classes is smaller than the maximum pixel distance within the class. That is, the circles are closer to each other than the diameter of the circle. Variation autoencoders, GANs and the original chart autoencoder all fail to properly represent the disconnected nature of this data-set, making conditional generation difficult. Divide-and-conquer methods struggle since the separation between the classes is very small. However, using only ten percent of the labels in a semi-supervised regime, we are able to successfully disentangle the dataset into 10 manifolds, each represented with two charts. As results, interpolation shown in Figure \ref{fig:Coil-10} illustrates that the proposed method successfully capture the overlapped data manifold structure. In table \ref{table:coil} we summarise the numerical results for several models using the metrics discussed in the previous section. Again, we observe that the chart-based models out preform the standard models, while also using fewer parameters.

\begin{table}[]
\centering
    \resizebox{.95\columnwidth}{!}{%
\begin{tabular}{c|c|c|ccc}
\hline
Method                    & \# of Charts          & Latent Dim & Reconstruction & Prediction & Generation \\ \hline
\multirow{2}{*}{Variational Autoencoder}      & \multirow{2}{*}{1} & 1          & 2.16E-01       & -                               & -          \\
                          &                    & 25         & 5.86E-03       & -                               & -          \\ \hline
\multirow{2}{*}{Conditional Autoencoder} & \multirow{2}{*}{1} & 1          & 9.89E-01       & -                               & 12.55      \\
                          &                    & 25         & 4.38E-03       & -                               & 88.32      \\ \hline
Divide and Conquer                       & 20                 & 1          & 2.11E-03       & 97.65                           & 95.46      \\ \hline
Chart Autoencoder                       & 20                 & 1          & 2.02E-03       & 98.36                           & 98.33      \\ \hline
\end{tabular}
}
\caption{Errors on Coil-10}
\label{table:coil}
\end{table}

In a second experiment on this dataset, we focus on one class and train latent functions to approximate the angle of the object's rotation. We test this using linear, 2-layer MLP and small ConvNet encoders. We repeat this experiment using linear and 3-layer MLP latent function models. The more complex latent models do a better job of predicting the angle since they are better able to counteract the distortion of the linear encoder. Finally, in table \ref{table:angleerror} we compare these results to similar networks trained on latent representations of a traditional auto-encoder and variational autoencoder and find ours favorable.

\begin{figure}[h]
  \includegraphics[width=\linewidth]{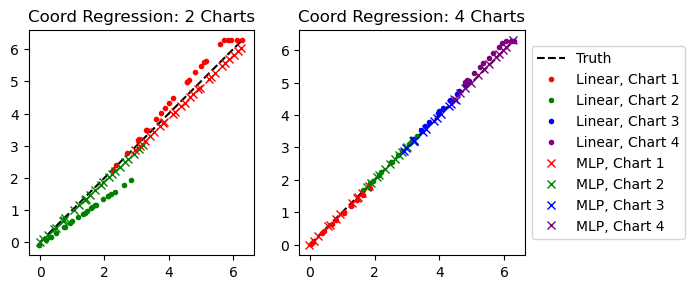}
  \caption{Regression of Angles in Coil with 2 and 4 chart models using linear and MLP latent function modules}
\end{figure}

\begin{table}
\centering
    \resizebox{.99\columnwidth}{!}{%
\begin{tabular}{c|c|ccccc}
\hline
Encoder   Model              & Latent Model & AE       & VAE      & Cond AE  & CAE-2    & CAE-4    \\ \hline
\multirow{2}{*}{Linear}      & Linear       & 1.57E+00 & 1.57E+00 & 8.18E-01 & 6.37E-02 & 6.54E-02 \\
                             & MLP          & 1.42E+00 & 1.32E+00 & 6.85E-01 & 4.73E-02 & 1.44E-02 \\ \hline
\multirow{2}{*}{2-Lyaer MLP} & Linear       & 7.71E-01 & 6.71E-01 & 3.49E-01 & 2.41E-02 & 1.91E-02 \\
                             & MLP          & 5.32E-01 & 3.32E-01 & 1.69E-01 & 4.32E-03 & 4.53E-03 \\ \hline
\multirow{2}{*}{Conv. Net}   & Linear       & 5.31E-02 & 2.14E-02 & 1.65E-02 & 9.55E-03 & 6.71E-03 \\
                             & MLP          & 9.16E-03 & 7.14E-03 & 5.15E-03 & 1.15E-03 & 1.32E-03 \\ \hline
\end{tabular}
    }
\caption{$L_1$ error for Coil-10 angle regression}
\label{table:angleerror}
\end{table}

\begin{figure}[ht]
    \centering
    \includegraphics[width=.70\linewidth]{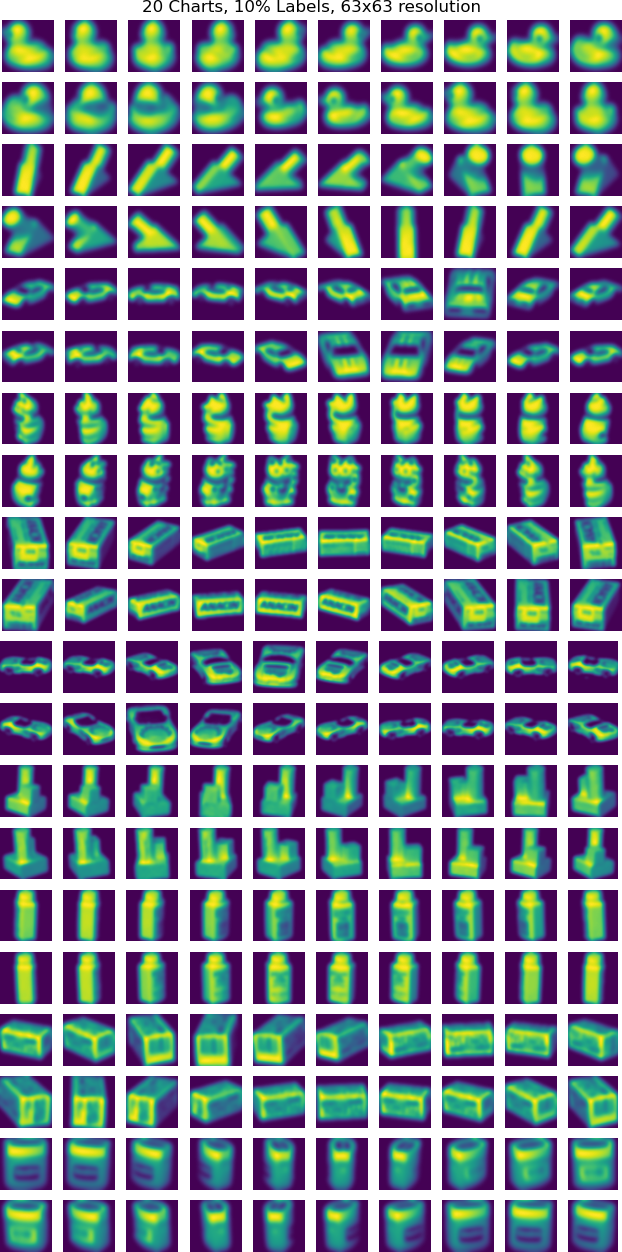}
    \caption{Sampling the latent space of a 20-chart CAE on Coil-10. Each row represents a single chart and charts are arranged by class}
    \label{fig:Coil-10}
\end{figure}

\subsection{Molecular Dynamics}

Finally, we show how chart autoencoders can be used in a real-world contex by modeling of molecular dynamics. The rMD17 dataset \cite{christensen2020role} contains hundreds of thousands of scans of 17 different molecules under various external conditions. Specifically, the data contains position of the atoms within a single molecule as the confirmation (i.e. pose) of the molecule changes due to inter atomic forces acting the molecule.  Since the entire dataset is not independently sampled (since the underlying sources is time series data), it also contains five standardized train-test splits of 1000 molecules each for property prediction which we employ in this study. 

Each molecular scan can be represented as a $N_{atom} \times 3$ matrix where each row represents the XYZ coordinate of an atom. We represent this observation as a vector $x_i \in \RR^{3 N_{atom}}$. However, it is well known that, even thought the $x_i$ can take on any value, the relative positions of the atoms with the molecule lie near a low-dimensional \textit{slow manifold} of dimension 2  \cite{koelle2022manifold, zhang2023dictionary}. Moreover, it is also known that the manifold is topologically toroidal, and cannot be isometircally embedded in less than four dimensions \cite{zhang2023dictionary}--that is, the slow manifold is well modeled by a flat torus. In this experiment, we are not only interest in modeling the position of the atoms, but also the the energy (measured in $\RR$) and molecular forces (measured as a vector in $\RR^3$). Each of these properties can be modeled as a function defined on the manifold of the coordinate positions. 

We compare our results with several with several competing models; a variation autoencoder (VAE \cite{kingma2013auto}, a conditional autoencoder (Cond AE) \cite{van2016conditional}, and a divide and conquer (D\&C) approach proposed \cite{floryan2022data}. For each method we employ two latent models. First, a simple linear regression, and second a relatively small 3-layer neural network (NN). We report additional network architectural details and training procedure in the appendix. Table \ref{table:Ethanol} reports the reconstruction of the test set coordinates and prediction of the molecular energy and forces for several latent space configurations. 

Several comments are in order. First, the methods which only employ a single chart (traditional and conditional auto encoders) achieve only lackluster results on the property prediction problem when the latent dimension is less than 4, the minimum dimension needed to isometrically embed the slow manifold even thought they have relatively good coordinate reconstruction scores. When the latent dimenions is 4, both the linear regressor and neural network based approach achieve good results. Overall the chart-based approaches (D\&C and CAE) achieve better results using only two dimensional latent spaces. Theoretically, the torus is homeomorphic to a two dimensional latent space (with periodic boundaries), however, both the divide and conqueror and CAE approaches fail to accurately model the dynamics when constrained to one or two charts. The failure of the D\&C method can be attributed to the fact that their KNN model for chart prediction does not guarantee that the charts are homeomphic to the latent space. Similarly, the CAE model struggles here because the finite distortion cannot be upheld with four or fewer charts. For both, increasing the number of charts alleviates these problems, and with 6 charts, both model outperforms the models with euclidean latent spaces. 

\begin{table}[]
\begin{tabular}{cccclccc}
\hline
\multicolumn{8}{c}{\textbf{Ethanol}}                                                                                                                                                                                                                                                                                                                                     \\ \hline
\multicolumn{1}{l|}{Model}                                                                        & \multicolumn{1}{l|}{Latent}              & \multicolumn{1}{l|}{\# of Charts}        & \multicolumn{1}{l|}{Latent Model} & \multicolumn{1}{l|}{\# of Parm}                   & \multicolumn{1}{c|}{Recon.} & \multicolumn{1}{l|}{Energy} & \multicolumn{1}{l}{Forces} \\ \hline
\multicolumn{1}{c|}{}                                                                             & \multicolumn{1}{c|}{}                    & \multicolumn{1}{c|}{}                    & \multicolumn{1}{c|}{Linear}       & \multicolumn{1}{l|}{7417}                         & 0.319914                    & 0.7809                      & 1.0781                     \\
\multicolumn{1}{c|}{}                                                                             & \multicolumn{1}{c|}{\multirow{-2}{*}{2}} & \multicolumn{1}{c|}{}                    & \multicolumn{1}{c|}{NN}           & \multicolumn{1}{l|}{9561}                         & 0.376983                    & 0.6531                      & 0.480863                   \\ \cline{2-2} \cline{4-8} 
\multicolumn{1}{c|}{}                                                                             & \multicolumn{1}{c|}{}                    & \multicolumn{1}{c|}{}                    & \multicolumn{1}{c|}{Linear}       & \multicolumn{1}{l|}{16342}                        & 0.283381                    & 0.7733                      & 0.8733                     \\
\multicolumn{1}{c|}{}                                                                             & \multicolumn{1}{c|}{\multirow{-2}{*}{3}} & \multicolumn{1}{c|}{}                    & \multicolumn{1}{c|}{NN}           & \multicolumn{1}{l|}{18486}                        & 0.309296                    & 0.6319                      & 0.371159                   \\ \cline{2-2} \cline{4-8} 
\multicolumn{1}{c|}{}                                                                             & \multicolumn{1}{c|}{}                    & \multicolumn{1}{c|}{}                    & \multicolumn{1}{c|}{Linear}       & \multicolumn{1}{l|}{23342}                        & 0.253628                    & 0.0191                      & 0.0591                     \\
\multicolumn{1}{c|}{\multirow{-6}{*}{VAE}}                                                        & \multicolumn{1}{c|}{\multirow{-2}{*}{4}} & \multicolumn{1}{c|}{\multirow{-6}{*}{1}} & \multicolumn{1}{c|}{NN}           & \multicolumn{1}{l|}{25486}                        & 0.241641                    & 0.0123                      & 0.173903                   \\ \hline
\multicolumn{1}{c|}{}                                                                             & \multicolumn{1}{c|}{}                    & \multicolumn{1}{c|}{1}                   & \multicolumn{1}{c|}{Linear}       & \multicolumn{1}{l|}{7417}                         & 0.29013                     & 0.735411                    & 0.95612                    \\
\multicolumn{1}{c|}{}                                                                             & \multicolumn{1}{c|}{\multirow{-2}{*}{2}} & \multicolumn{1}{c|}{}                    & \multicolumn{1}{c|}{NN}           & \multicolumn{1}{l|}{9561}                         & 0.365567                    & 0.660253                    & 0.474884                   \\ \cline{2-2} \cline{4-8} 
\multicolumn{1}{c|}{}                                                                             & \multicolumn{1}{c|}{}                    & \multicolumn{1}{c|}{}                    & \multicolumn{1}{c|}{Linear}       & \multicolumn{1}{l|}{16342}                        & 0.262462                    & 0.775172                    & 0.882703                   \\
\multicolumn{1}{c|}{}                                                                             & \multicolumn{1}{c|}{\multirow{-2}{*}{3}} & \multicolumn{1}{c|}{}                    & \multicolumn{1}{c|}{NN}           & \multicolumn{1}{l|}{18486}                        & 0.28925                     & 0.642878                    & 0.361031                   \\ \cline{2-2} \cline{4-8} 
\multicolumn{1}{c|}{}                                                                             & \multicolumn{1}{c|}{}                    & \multicolumn{1}{c|}{}                    & \multicolumn{1}{c|}{Linear}       & \multicolumn{1}{l|}{23342}                        & 0.230159                    & 0.019287                    & 0.060499                   \\
\multicolumn{1}{c|}{\multirow{-6}{*}{\begin{tabular}[c]{@{}c@{}}Cond.   \\      AE\end{tabular}}} & \multicolumn{1}{c|}{\multirow{-2}{*}{4}} & \multicolumn{1}{c|}{\multirow{-5}{*}{}}  & \multicolumn{1}{c|}{NN}           & \multicolumn{1}{l|}{25486}                        & 0.241433                    & 0.011554                    & 0.163288                   \\ \hline
\multicolumn{1}{c|}{}                                                                             & \multicolumn{1}{c|}{}                    & \multicolumn{1}{c|}{}                    & \multicolumn{1}{c|}{Linear}       & \multicolumn{1}{l|}{6824}                         & 0.283381                    & 0.716232                    & 0.993962                   \\
\multicolumn{1}{c|}{}                                                                             & \multicolumn{1}{c|}{}                    & \multicolumn{1}{c|}{\multirow{-2}{*}{2}} & \multicolumn{1}{c|}{NN}           & \multicolumn{1}{l|}{7944}                         & 0.309296                    & 0.631924                    & 0.455209                   \\ \cline{3-8} 
\multicolumn{1}{c|}{}                                                                             & \multicolumn{1}{c|}{}                    & \multicolumn{1}{c|}{}                    & \multicolumn{1}{c|}{Linear}       & \multicolumn{1}{l|}{13602}                        & 0.274378                    & 0.0373                      & 0.1119                     \\
\multicolumn{1}{c|}{}                                                                             & \multicolumn{1}{c|}{}                    & \multicolumn{1}{c|}{\multirow{-2}{*}{4}} & \multicolumn{1}{c|}{NN}           & \multicolumn{1}{l|}{15896}                        & 0.269803                    & 0.0235                      & 0.061145                   \\ \cline{3-8} 
\multicolumn{1}{c|}{}                                                                             & \multicolumn{1}{c|}{}                    & \multicolumn{1}{c|}{}                    & \multicolumn{1}{c|}{Linear}       & \multicolumn{1}{l|}{20550}                        & 0.294648                    & 0.0213                      & 0.0639                     \\
\multicolumn{1}{c|}{\multirow{-6}{*}{D+C}}                                                        & \multicolumn{1}{c|}{\multirow{-6}{*}{2}} & \multicolumn{1}{c|}{\multirow{-2}{*}{6}} & \multicolumn{1}{c|}{NN}           & \multicolumn{1}{l|}{23856}                        & 0.289832                    & 0.0195                      & 0.090632                   \\ \hline
\multicolumn{1}{c|}{}                                                                             & \multicolumn{1}{c|}{}                    & \multicolumn{1}{c|}{}                    & \multicolumn{1}{c|}{Linear}       & \multicolumn{1}{l|}{6878}                         & 0.274378                    & 0.0373                      & 0.1119                     \\
\multicolumn{1}{c|}{}                                                                             & \multicolumn{1}{c|}{}                    & \multicolumn{1}{c|}{\multirow{-2}{*}{2}} & \multicolumn{1}{c|}{NN}           & \multicolumn{1}{l|}{7998}                         & 0.269803                    & 0.0235                      & 0.061145                   \\ \cline{3-8} 
\multicolumn{1}{c|}{}                                                                             & \multicolumn{1}{c|}{}                    & \multicolumn{1}{c|}{}                    & \multicolumn{1}{c|}{Linear}       & \multicolumn{1}{l|}{13764}                        & 0.294648                    & 0.0213                      & 0.0639                     \\
\multicolumn{1}{c|}{}                                                                             & \multicolumn{1}{c|}{}                    & \multicolumn{1}{c|}{\multirow{-2}{*}{4}} & \multicolumn{1}{c|}{NN}           & \multicolumn{1}{l|}{{\color[HTML]{000000} 16004}} & 0.289832                    & 0.0195                      & 0.090632                   \\ \cline{3-8} 
\multicolumn{1}{c|}{}                                                                             & \multicolumn{1}{c|}{}                    & \multicolumn{1}{c|}{}                    & \multicolumn{1}{c|}{Linear}       & \multicolumn{1}{l|}{20658}                        & 0.227143                    & 0.0183                      & 0.05673                    \\
\multicolumn{1}{c|}{\multirow{-6}{*}{CAE}}                                                        & \multicolumn{1}{c|}{\multirow{-6}{*}{2}} & \multicolumn{1}{c|}{\multirow{-2}{*}{6}} & \multicolumn{1}{c|}{NN}           & \multicolumn{1}{l|}{24018}                        & 0.211678                    & 0.0079                      & 0.011169                   \\ \hline
\end{tabular}
\caption{Molecular dynamics experiment results for ethanol dataset}
\label{table:Ethanol}
\end{table}

\section{Conclusion}

In this work, we have proposed a chart auto-encoder that can incorporate additional information such as class labels to detect hidden low-dimensional manifold structures in a semi-supervised manner. The proposed asymmetric encoding-decoding process enjoys a very low complexity encoder, such as local linear projection. Meanwhile, it allows us to learn representations for manifolds with highly nontrivial topology, including manifolds with holes and disconnected manifolds. In addition, it can successfully differentiate nearby but disjoint manifolds and intersecting manifolds with only a small amount of supervision. Theoretically, we show that such approximations can be accurately conducted with neural networks of width and depth that scale with the intrinsic dimension of the data and only a weak dependence on the ambient dimension. 
Our numerical experiments on synthetic and real-world data demonstrate the effectiveness of the proposed method. Future work may involve approximating more complicated processes, such as dynamical systems on manifolds, or adopting this chart auto-encoder to other generative models, such as adversarial networks or diffusion networks. 

\section{Acknowledgements}
R. Lai's reserach is supported in part by NSF DMS-2401297.  A. Cloninger's research is supported in part by NSF DMS-2012266, CISE-2403452, and a gift from Intel.

\bibliographystyle{plain} 
\bibliography{refs_jmlr}

\newpage
\appendix
\onecolumn

\section{Architecture Details for Numerical Results} \label{app:models}

\textbf{Single Chart CAE- Figure 2} 

\begin{equation*}\label{eqn:model.cae0}
  \begin{split}
  \textbf{Number of charts} = 1 \\
    \text{Chart Encoder}  &: x \rightarrow FC_{10} \rightarrow (FC_{d},FC_{d}) \rightarrow z_{\alpha}, \sigma_{\alpha} \\
    \text{Chart Decoders}  &: z_{\alpha} \rightarrow  FC_{10} \rightarrow  FC_{10} \rightarrow FC_{2d} \rightarrow       y_{\alpha} \\
    \text{Linear Latent}&: z \rightarrow FC_{1} \rightarrow f \\
    \text{Chart Prediction}&: x \rightarrow FC_{10} \rightarrow FC_{10}  \rightarrow  FC_{N} \rightarrow \\
        & softmax \rightarrow p
  \end{split}
\end{equation*}

\textbf{Geometric CAE- Figures 4,5,6} 

\begin{equation*}\label{eqn:model.cae1}
  \begin{split}
  \textbf{Number of charts} = 5,6,12 \\
    \text{Chart Encoder}  &: x \rightarrow FC_{10} \rightarrow \\
    & (FC_{d},FC_{d}) \rightarrow  z_{\alpha}, \sigma_{\alpha} \\
    \text{Chart Decoders}  &: z_{\alpha} \rightarrow  FC_{10} \rightarrow  FC_{10} \\
        &\rightarrow FC_{2d}  \rightarrow       y_{\alpha} \\
    \text{Constant Latent}&: z_{\alpha} \rightarrow  f_{\alpha} \\
    \text{Chart Prediction}&: x \rightarrow FC_{10} \rightarrow FC_{10}  \rightarrow \\
       & FC_{N} \rightarrow softmax \rightarrow p
  \end{split}
\end{equation*}

\textbf{Coil-10 CAE- Figure 8}
\begin{equation*}\label{eqn:model.cae3}
  \begin{split}
  \textbf{Number of charts} = 20 \\
    \text{Chart Encoder}&: x \rightarrow FC_{100} \rightarrow (FC_{d},FC_{d}) \rightarrow \\
        & z_{\alpha}, \sigma_{\alpha} \\
    \text{Chart Decoder}&:z_{\alpha}\rightarrow Conv_{3,3,16,8} \\
        & \rightarrow  Conv_{3,3,8,8} \rightarrow  Conv_{3,3,8,1} \\
        & \rightarrow  y_{\alpha}, \sigma_{\alpha} \\
    \text{Constant Latent }&: z_{\alpha} \rightarrow f_{\alpha} \\
    \text{Chart Prediction}&: z \rightarrow FC_{100} \rightarrow \\
        & FC_{100} \rightarrow softmax \rightarrow p
  \end{split}
\end{equation*}

\textbf{Coil-1 with angle regression CAE- Figure 7}
\begin{equation*}\label{eqn:model.coil 1}
  \begin{split}
    \textbf{Number of charts} = 2,4 \\
    \text{Chart Encoder}&: x \rightarrow (FC_{1},FC_{1}) \rightarrow z_{\alpha}, \sigma_{\alpha} \\  
    \text{Chart Decoder}&:z_{\alpha}\rightarrow Conv_{3,3,16,64} \\
        & \rightarrow  Conv_{3,3,64,8} \rightarrow  Conv_{3,3,8,1} \\
        & \rightarrow FC_{m} \rightarrow  y_{\alpha} \\
    \text{Linear Latent}&: z \rightarrow  FC_{1} \rightarrow 2\pi * sigmoid \rightarrow \\
        f_{\alpha} \\
    \text{MLP Latent}&: z \rightarrow  FC_{10} \rightarrow  FC_{10} \\ 
        & \rightarrow  FC_{2} \rightarrow 2\pi * sigmoid \rightarrow  f_{\alpha} \\
    \text{Chart Prediction}&: z \rightarrow FC_{100} \rightarrow \\
        & FC_{100} \rightarrow softmax \rightarrow p
  \end{split}
\end{equation*}

\section{Proofs}

\subsection{Proof of Theorem \ref{thm:reach}} \label{app:reach}

Theorem \ref{thm:reach} holds by applying Theorem \ref{thm:reach2} to each chart $U_i$. The characterization of reach as the radius of the largest ambient tangent ball is used in our proof of Theorem \ref{thm:reach2} below.

\begin{theorem} \label{thm:reach2}
    Let $\delta>0$. Suppose $U \subset \real^D$ is a connected $d$-dimensional manifold such that $\|u_1-u_2\| \leq 2\tau_{U} - \delta$ for all $u_1,u_2 \in U$. Then there is a $d$-dimensional subspace $\H \subset \real^D$ such that the projection map $\pi_{\H} : U \to \H$ is bi-Lipschitz with lower Lipschitz constant $\big( 1 + (D-d)\tfrac{2\tau_U}{\delta}\big)^{-1/2}$.
\end{theorem}

\begin{proof}

We first proof the theorem in the case of a codimension 1 manifold, then leverage this result in the general case.

\paragraph{Case 1:} \textit{Suppose $U$ is $(D-1)$-dimensional.}

Recall from the geometric intuition of reach in Figure \ref{fig:reach} that an ambient ball of radius $\tau_U$ that lies tangent to $U$ does not contain any points from $U$ in its interior. Moreover, since $\|u_1-u_2\| \leq 2\tau_U - \delta$ for all $u_1,u_2 \in U$, we can choose an ambient ball $\mathcal{B}$ of radius $\tau_U$ tangent to $U$ such that $U$ lies on one side of an affine hyperplane $\mathcal{H}$ that intersects $\mathcal{B}$ to create a spherical cap whose base has a diameter of $2\tau_U - \delta$, as shown in Figure \ref{fig:reach2}. 

\begin{figure}[H]
\centering
\scalebox{2.5}{\begin{tikzpicture}
    \draw [black] plot[smooth,domain=-1.125:1.125] (\x, {-(\x)^2/2});
    \draw [dashed, blue] (0,-1) circle (1);
    \filldraw (0,-1) circle (1pt);
    \draw (-0.94,-0.7) -- (0.94,-0.7);
    \draw [dashed] (-1.4,-0.7) -- (1.4,-0.7);
    \draw [dashed, blue] (-1,-1) -- (1,-1);
    \node (T) [blue] at (0.45,-1.15) {\tiny $2\tau_U$};
    \node (D) at (0.35,-0.55) {\tiny $2\tau_U-\delta$};
    \node (U) at (1.2,-0.5) {\tiny $U$};
    \node (H) at (-1.4,-0.8) {\tiny $\mathcal{H}$};
    \node (B) [blue] at (1.0,-1.5) {\tiny $\mathcal{B}$};
    
    \draw [densely dotted] (0,-1) -- (-0.94,-0.7);
    \draw [densely dotted] (-0.94,-1) -- (-0.94,-0.7);
    \node (h) at (-0.85,-0.85) {\tiny $h$};
    \draw [dashdotted] (-0.95,-0.7) -- (-0.6,0.412929);
    \draw [dashdotted] (-0.95,-0.7) -- (-1.4,-2.130909);
    
    \filldraw (-1.05,-0.55125) circle (1pt);
    \filldraw (0,0) circle (1pt);
    \draw [densely dotted] (-1.05,-0.55125) -- (0,0);
    \node (u1) at (-1.23,-0.45) {\tiny $u_1$};
    \node (u2) at (0,0.15) {\tiny $u_2$};
\end{tikzpicture}}
\caption{A 2-dimensional slice of an ambient tangent ball of radius $\tau_U$, chosen so that $U$ lies above the hyperplane $\mathcal{H}$ which intersects $\mathcal{B}$ with chord length $2\tau_U - \delta$.}
\label{fig:reach2}
\end{figure}

Let $\{v_1,\dots,v_{D-1}\}$ be an orthonormal basis for $\mathcal{H}$, and let $\bm{n}$ be a unit vector normal to $\mathcal{H}$ so that $\{v_1,\dots,v_{D-1},\bm{n}\}$ is an orthonormal basis for $\real^D$. Given $u_1,u_2 \in U$, consider a 2-dimensional slice of $\mathcal{B}$ and $U$ parallel to $\pi_{\mathcal{H}}(u_1)-\pi_{\mathcal{H}}(u_2)$ and perpendicular to $\mathcal{H}$, as pictured in Figure \ref{fig:reach2}. The slope of the secant line between $u_1$ and $u_2$ is $|\langle u_1-u_2, \bm{n} \rangle| /  \| \pi_{\mathcal{H}}(u_1)-\pi_{\mathcal{H}}(u_2) \|$. Moreover, the slope of this secant line is smaller in absolute value than the slope of the line tangent to $\mathcal{B}$ at its intersection with $\mathcal{H}$. Indeed, if the slope of the secant line exceeded that of the tangent line, then there would be a geodesic $\gamma: [a,b] \to U$ of $U$ such that $\|\gamma''(t)\| > \tau_U^{-1}$ for some $t$ (i.e., the curvature of $U$ would exceed $\tau_U^{-1}$ at some point which would mean that $\mathcal{B}$ could not lie tangent to $U$ at this point, contradicting the definition of $\tau_U$). This tangent line has slope $(\tau_U - \delta/2)/h$, where $h$ is the distance between $\mathcal{H}$ and the parallel hyperplane that divides $\mathcal{B}$ in half. Note that
\begin{equation*}
    h = \sqrt{\tau_U^2 - (\tau_U - \delta/2)^2} = \sqrt{\tau_U^2 - (\tau_U^2 - \delta \tau_U + \delta^2/4)} = \sqrt{\delta \tau_U - \delta^2/4}.
\end{equation*}
Moreover, since $\delta < 2\tau_U$ we have $\tau_U - \delta/4 > \tau_U - \tau_U/2 = \tau_U/2$, so
\begin{equation*}
    h = \sqrt{\delta \tau_U - \delta^2/4} = \sqrt{\delta (\tau_U - \delta/4)} > \sqrt{\delta\tau_U/2}.
\end{equation*}
Since the slope of the secant line between $u_1$ and $u_2$ is bounded by the slope of the aforementioned tangent line, we have
\begin{equation*}
    \frac{|\langle u_1-u_2, \bm{n} \rangle| }{  \| \pi_{\mathcal{H}}(u_1)-\pi_{\mathcal{H}}(u_2) \|} \leq \frac{\tau_U - \delta/2}{h} < \frac{\tau_U}{\sqrt{\delta\tau_U/2}} = \left( \frac{2\tau_U}{\delta}\right)^{1/2}.
\end{equation*}
Finally, since $\{v_1,\dots,v_{D-1},\bm{n}\}$ is an orthonormal basis for $\real^D$, it follows that
\begin{align*}
    \|u_1-u_2\|^2 &= \sum_{i=1}^{D-1} |\langle u_1-u_2, v_i \rangle|^2 + |\langle u_1-u_2, \bm{n} \rangle|^2 \\
    &= \| \pi_{\mathcal{H}}(u_1)-\pi_{\mathcal{H}}(u_2) \|^2 + |\langle u_1-u_2, \bm{n} \rangle|^2 \\
    &\leq \| \pi_{\mathcal{H}}(u_1)-\pi_{\mathcal{H}}(u_2) \|^2  + \left( \frac{2\tau_U}{\delta}\right) \| \pi_{\mathcal{H}}(u_1)-\pi_{\mathcal{H}}(u_2) \|^2 \\
    &= \Big( 1 + \frac{2\tau_U}{\delta}\Big) \| \pi_{\mathcal{H}}(u_1)-\pi_{\mathcal{H}}(u_2) \|^2.
\end{align*}
Therefore $\left( 1 + \tfrac{2\tau_U}{\delta}\right)^{-1/2} \|u_1-u_2\| \leq  \| \pi_{\mathcal{H}}(u_1)-\pi_{\mathcal{H}}(u_2) \| \leq \|u_1-u_2\|$, so $\pi_{\mathcal{H}}$ is bi-Lipschitz.

\paragraph{Case 2:} \textit{Suppose $U$ is $d$-dimensional with $1 \leq d < D-1$.}

Begin by choosing an ambient ball $\mathcal{B}$ of radius $\tau_U$ that lies tangent to $U$ such that $U$ lies on one side of a $(D-1)$-dimensional hyperplane $\mathcal{H}_1$ that intersects $\mathcal{B}$ to create a spherical cap whose base has a diameter of $2\tau_U - \delta$, as in Case 1 (this is still possible because of the condition $\|u_1-u_2\| \leq 2\tau_U - \delta$ for all $u_1,u_2 \in U$). Let $\bm{n}_1$ denote the unit normal vector of $\mathcal{H}_1$. Next, choose any $(D-1)$-dimensional hyperplane $\hat{\mathcal{H}}_2$ that is orthogonal to $\mathcal{H}_1$. Again by the condition $\|u_1-u_2\| \leq 2\tau_U - \delta$ for all $u_1,u_2 \in U$, there exists an orthogonal matrix $Q_1$ such that $Q_1\bm{n}_1 = \bm{n}_1$ and $U$ lies on one side of $\mathcal{H}_2 = Q_1\hat{\mathcal{H}}_2$. Let $\bm{n}_2$ denote the unit normal vector of $\mathcal{H}_2$. Continuing in this fashion, we can construct a sequence $\mathcal{H}_1, \dots, \mathcal{H}_{D-d}$ of mutually orthogonal hyperplanes of dimension $D-1$ that each intersect an ambient ball tangent to $U$ with a spherical cap diameter of $2\tau_U-\delta$. Let $\bm{n}_1,\dots,\bm{n}_{D-d}$ denote the unit normal vectors of these hyperplanes. Next, define $$\mathcal{H} = \bigcap_{k=1}^{D-d} \mathcal{H}_{k}.$$ Since each $\mathcal{H}_{k}$ is $(D-1)$-dimensional and orthogonal, $\mathcal{H}$ is $d$-dimensional. Let $\{v_1,\dots,v_d\}$ be an orthonormal basis for $\mathcal{H}$ so that $\{v_1,\dots,v_d,\bm{n}_1,\dots,\bm{n}_{D-d}\}$ is an orthonormal basis for $\real^D$. Given $u_1,u_2 \in U$ and $k \in \{1,\dots,D-d\}$, consider the 2-dimensional slice of $\real^D$ parallel to $\pi_{\mathcal{H}}(u_1)-\pi_{\mathcal{H}}(u_2)$ and $\bm{n}_k$. As in Case 1, we see that $|\langle u_1-u_2, \bm{n}_k \rangle|^2 \leq \left( \frac{2\tau_U}{\delta}\right)  \| \pi_{\mathcal{H}}(u_1)-\pi_{\mathcal{H}}(u_2) \|^2$. It follows that
\begin{align*}
    \|u_1-u_2\|^2 &= \sum_{i=1}^{d} |\langle u_1-u_2, v_i \rangle|^2 + \sum_{k=1}^{D-d} |\langle u_1-u_2, \bm{n}_k \rangle|^2 \\
    &= \| \pi_{\mathcal{H}}(u_1)-\pi_{\mathcal{H}}(u_2) \|^2 + \sum_{k=1}^{D-d} |\langle u_1-u_2, \bm{n}_k \rangle|^2 \\
    &\leq \| \pi_{\mathcal{H}}(u_1)-\pi_{\mathcal{H}}(u_2) \|^2 \\
    & \quad \quad + (D-d)\Big( \frac{2\tau_U}{\delta}\Big) \| \pi_{\mathcal{H}}(u_1)-\pi_{\mathcal{H}}(u_2) \|^2 \\
    &= \Big( 1 + (D-d)\frac{2\tau_U}{\delta}\Big) \| \pi_{\mathcal{H}}(u_1)-\pi_{\mathcal{H}}(u_2) \|^2.
\end{align*}
Therefore $\left( 1 + (D-d)\tfrac{2\tau_U}{\delta}\right)^{-1/2} \|u_1-u_2\| \leq  \| \pi_{\mathcal{H}}(u_1)-\pi_{\mathcal{H}}(u_2) \| \leq \|u_1-u_2\|$, so $\pi_{\mathcal{H}}$ is bi-Lipschitz.
\end{proof}

\subsection{Proof of Theorem \ref{thm:gauss}} \label{app:gauss}

Theorem \ref{thm:gauss} holds by applying the following result to each chart. We remark again that a similar result may hold for non-smooth manifolds of codimension greater than one.

\begin{theorem} 
    Let $\mathcal{M}$ be a smooth orientable $(D-1)$-dimensional manifold in $\R^D$ satisfying $d_\mathcal{M}(u_1,u_2)^2 \leq C\|u_1-u_2\|^2$ (with $C>1$) for all $u_1,u_2 \in \mathcal{M}$, where $d_\mathcal{M}$ is the geodesic distance on $\mathcal{M}$. Suppose $U \subset \mathcal{M}$ is geodesically convex and there exists a unit vector $n \in \R^D$ such that $\langle N(u), n \rangle \geq \delta \geq \sqrt{1-1/C}$ for all $u \in U$, where $N: U \to \mathcal{S}^{D-1}$ is the Gauss map of $U$. Then the linear projection $\pi_\mathcal{H} : U \to \mathcal{H}$ is bi-Lipschitz, where $\mathcal{H} = \{x \in \R^D \, | \, \langle x, n \rangle = 0\}$. 
\end{theorem}

\begin{proof}
    Let $v_1,\dots,v_{D-1}$ be an orthonormal basis for $\mathcal{H}$. Then
    \begin{equation*}
        \pi_\mathcal{H}(u) = \sum_{i=1}^{D-1} \langle u, v_i \rangle v_i
    \end{equation*}
    for $u \in U$. For $u_1,u_2 \in U$,
    \begin{equation*}
        \|\pi_\mathcal{H}(u_1) - \pi_\mathcal{H}(u_2)\|^2 = \sum_{i=1}^{D-1} |\langle u_1-u_2, v_i \rangle|^2
    \end{equation*}
    and
    \begin{equation*}
        \|u_1 - u_2\|^2 = \sum_{i=1}^{D-1} |\langle u_1-u_2, v_i \rangle|^2 + |\langle u_1-u_2, n \rangle|^2 = \|\pi_\mathcal{H}(u_1) - \pi_\mathcal{H}(u_2)\|^2 + |\langle u_1-u_2, n \rangle|^2
    \end{equation*}
    since $v_1,\dots,v_{D-1},n$ is an orthonormal basis for $\R^D$. Now we will bound $|\langle u_1-u_2, n \rangle|^2$. First, let $\gamma:[0,\ell] \to U$ be a unit speed geodesic from $u_1$ to $u_2$, i.e. $\gamma(0)=u_1$, $\gamma(\ell)=u_2$, and $\|\gamma'(t)\|=1$. Define $f:[0,\ell] \to \R$ by
    \begin{equation*}
        f(t) = \langle \gamma(t), n \rangle = \sum_{i=1}^D \gamma_i(t) n_i.
    \end{equation*}
    Note that
    \begin{equation*}
        f'(t) = \sum_{i=1}^D \gamma_i'(t) n_i = \langle \gamma'(t), n \rangle.   
    \end{equation*}
    Since $f$ is differentiable, by the Mean Value Theorem there is $c \in [0,\ell]$ such that
    \begin{equation*}
        f(\ell)-f(0) = f'(c) (\ell-0).
    \end{equation*}
    Note that $|f(\ell)-f(0)| = |\langle u_1-u_2, n \rangle|$ and $\ell=d_\mathcal{M}(u_1,u_2)$. Moreover,
    \begin{equation*}
        |f'(c)|^2 = |\langle \gamma'(c), n \rangle|^2 \leq 1 - |\langle N(\gamma(c)), n \rangle|^2 \leq 1 - \delta^2 
    \end{equation*}
    since $|\langle \gamma'(c), n \rangle|^2 + |\langle N(\gamma(c)), n \rangle|^2 \leq \|n\|^2 = 1$. Therefore
    \begin{equation*}
        |\langle u_1-u_2, n \rangle|^2 = |f(\ell)-f(0)|^2 \leq |f'(c)|^2 \ell^2 \leq (1 - \delta^2) d_\mathcal{M}(u_1,u_2)^2 \leq C (1 - \delta^2) \|u_1-u_2\|^2.
    \end{equation*}
    It follows that
    \begin{align*}
        \|u_1-u_2\|^2 &= \|\pi_\mathcal{H}(u_1) - \pi_\mathcal{H}(u_2)\|^2 + |\langle u_1-u_2, n \rangle|^2 \\
        &\leq \|\pi_\mathcal{H}(u_1) - \pi_\mathcal{H}(u_2)\|^2 + C(1-\delta^2) \|u_1-u_2\|^2
    \end{align*}
    and finally $\sqrt{1-C(1-\delta^2)} \|u_1-u_2\| \leq \|\pi_\mathcal{H}(u_1) - \pi_\mathcal{H}(u_2)\| \leq \|u_1-u_2\|$. Since $\delta > \sqrt{1-1/C}$ we have $\sqrt{1-C(1-\delta^2)} > \sqrt{1-C(1/C)} = 0$, and therefore $\pi_\mathcal{H}$ is bi-Lipschitz.
\end{proof}

\subsection{Proof of Theorem \ref{thm:decoder_complexity}} \label{app:decoder_complexity}

Since the atlas $\{(U_i,\phi_i)\}_{i=1}^n$ is a finite distortion embedding of $\M$, each decoder $D_i = \phi_i^{-1}|_{\phi_i(U_i)} : \phi_i(U_i) \to \real^D$ is Lipschitz. 
Thus, Theorem \ref{thm:decoder_complexity} holds by the following result, with the additional observation that only the final layer depends on the ambient dimension $D$ of the manifold.

\begin{theorem}
    Let $f: \real[0,1]^d \to \real^D$ be bounded and Lipschitz. 
    Then for any $\epsilon \in (0,1)$, there is a ReLU network $\tilde{f}$ such that
    \begin{enumerate}
        \item $\|f-\tilde{f}\|_\infty < \epsilon$;
        \item $\tilde{f}$ has at most $c\ln(1/\epsilon)+1$ layers and $c\epsilon^{-d}(\ln(1/\epsilon)+D)$ computation units, for some constant $c=c(d)$.
    \end{enumerate}
\end{theorem}

\begin{proof}
    For a positive integer $N$, let $\{\varphi_{\bm{m}}\}_{\bm{m} \in \{0,1,\dots,N\}^d}$ be the partition of unity on the domain $[0,1]^d$ defined on a grid of $(N+1)^d$ points as in \cite{yarotsky2017error}:
    \begin{equation}
        \varphi_{\bm{m}}(\bm{x}) = \prod_{k=1}^d \psi(3Nx_k - m_k),
    \end{equation}
    where
    \begin{equation}
        \psi(x) = \begin{cases}
            1,      &|x| < 1,           \\
            0,      &2 < |x|,           \\
            2-|x|,  &1 \leq |x| \leq 2.
        \end{cases}
    \end{equation}
    Define
    \begin{equation}
        \hat{f} = \sum_{\bm{m}} \varphi_{\bm{m}} f\left(\frac{\bm{m}}{N}\right),
    \end{equation}
    a sum-product of the partition of unity with constant functions equal to $f$ evaluated at each grid point $\bm{x} = \frac{\bm{m}}{N}$. Then
    \begin{align*}
        \|f(x) - \hat{f}(x)\| &= \Big\| \sum_{\bm{m}} \varphi_{\bm{m}}(\bm{x}) \big( f(\bm{x}) -  f(\tfrac{\bm{m}}{N}) \big) \Big\| \\[+2mm]
        &\leq \sum_{\bm{m}} |\varphi_{\bm{m}}(\bm{x})| \|f(\bm{x}) -  f(\tfrac{\bm{m}}{N})\| \\[+2mm]
        &\leq \sum_{\bm{m}: |x_k-\frac{m_k}{N}| < \frac{1}{N} \forall k} \|f(\bm{x}) -  f(\tfrac{\bm{m}}{N})\| \\[+2mm]
        &\leq 2^d \max_{\bm{m}: |x_k-\frac{m_k}{N}| < \frac{1}{N} \forall k} \|f(\bm{x}) -  f(\tfrac{\bm{m}}{N})\|
    \end{align*}
    since $\text{supp}~\varphi_{\bm{m}} = \big\{ \bm{x} : \big|x_k-\tfrac{m_k}{N}\big|<\tfrac{1}{N} ~\forall k \big\}$ and each $\bm{x}$ belongs to the support of at most $2^d$ functions $\varphi_{\bm{m}}$. Applying Taylor's theorem to $\|f(\bm{x}) -  f(\tfrac{\bm{m}}{N})\|$, we get
    \begin{equation}
        \|f(x) - \hat{f}(x)\| \leq 2^d \frac{\sqrt{d}}{N} \esssup_{\bm{x} \in [0,1]^d} \|J_f(\bm{x})\| \leq \frac{2^d L \sqrt{d}}{N}
    \end{equation}
    where $L$ is the Lipschitz constant of $f$, since $\|\bm{x} - \tfrac{\bm{m}}{N}\| \leq \tfrac{\sqrt{d}}{N}$ for all $\bm{x} \in \text{supp}~\varphi_{\bm{m}}$. Thus, if we choose
    \begin{equation} \label{eq:N}
        N > \frac{2^{d+1}L \sqrt{d}}{\epsilon}
    \end{equation}
    then
    \begin{equation}
        \|f - \hat{f}\|_\infty \leq \frac{\epsilon}{2}.
    \end{equation}
    We will now approximate
    \begin{equation*}
        \hat{f}(\bm{x}) = \sum_{\bm{m}} \varphi_{\bm{m}}(\bm{x}) f\left(\frac{\bm{m}}{N}\right)
    \end{equation*}
    with a neural network. By Proposition 3 in \cite{yarotsky2017error}, given $K>0$ and $\delta \in (0,1)$, there exists a ReLU network $\eta$ that implements a ``multiplication'' function $\tilde{\times} : \real^2 \to \real$ such that $|\tilde{\times}(x,y) - xy| \leq \delta$ whenever $|x| \leq K$ and $|y| \leq K$, $\tilde{\times}(x,y)=0$ whenever $x=0$ or $y=0$, and the number of layers and computation units in $\eta$ is at most $c_1\ln(1/\delta) + c_2$, where $c_2 = c_2(K)$. For each $\bm{m} \in \{0,1,\dots,N\}^d$, define
    \begin{equation} \label{eq:ftilde}
        \tilde{\varphi}_{\bm{m}}(\bm{x}) = \tilde{\times} \big( \psi(3Nx_1-m_1), \tilde{\times}(\psi(3Nx_2-m_2), \dots ) \big), 
    \end{equation}
    which includes $d-1$ applications of $\tilde{\times}$. Since $|\psi(x)| \leq 1$, repeated application of the property $|\tilde{\times}(x,y)| \leq |x| \cdot |y| + \delta$ with $|x| \leq 1$ and $\delta < 1$ shows that each argument given to $\tilde{\times}$ in Equation \eqref{eq:ftilde} is bounded by $K=d$. It follows that $\tilde{\varphi}_{\bm{m}}$ can be implemented by a ReLU network with number of layers and units at most $c \ln(1/\delta)$ for some constant $c = c(d)$. Moreover, we have
    \begin{align*}
        |\tilde{\varphi}_{\bm{m}}(\bm{x}) - \varphi_{\bm{m}}(\bm{x})| &= \left| \tilde{\times} \big( \psi(3Nx_1-m_1), \tilde{\times}(\psi(3Nx_2-m_2), \dots ) \big) - \psi(3Nx_1-m_1) \psi(3Nx_2-m_2) \cdots \right| \\
        &\leq \left| \tilde{\times} \big( \psi(3Nx_1-m_1), \tilde{\times}(\psi(3Nx_2-m_2), \tilde{\times}(\psi(3Nx_3-m_3) \dots ) \big) \right. \\
        &\hspace{5.8mm}\left. - ~\psi(3Nx_1-m_1) \cdot \tilde{\times}\big(\psi(3Nx_2-m_2), \tilde{\times}(\psi(3Nx_3-m_3), \dots \big) \right| \\
        &\hspace{5.8mm} + |\psi(3Nx_1-m_1)| \cdot \left| \tilde{\times}\big(\psi(3Nx_2-m_2), \tilde{\times}(\psi(3Nx_3-m_3), \dots \big) \right. \\
        &\hspace{5.8mm}\left. - ~ \psi(3Nx_2-m_2) \cdot \tilde{\times}\big(\psi(3Nx_3-m_3), \dots \big) \right| \\
        &\hspace{5.8mm} + \cdots \\
        &\leq d \delta
    \end{align*}
    since $|\tilde{\times}(x,y) - xy| \leq \delta$ and $|\psi(x)| \leq 1$. Additionally, $\tilde{\varphi}_{\bm{m}}(\bm{x}) = \varphi_{\bm{m}}(\bm{x}) = 0$ for $\bm{x} \not \in \text{supp}\, \varphi_{\bm{m}}$ since $\tilde{\times}(x,y)=0$ whenever $x=0$ or $y=0$. Thus, if we define
    \begin{equation} \label{eq:NN}
        \tilde{f}(\bm{x}) = \sum_{\bm{m}} \tilde{\varphi}_{\bm{m}}(\bm{x}) f\left(\frac{\bm{m}}{N}\right)
    \end{equation}
    then
    \begin{align*}
        \| \tilde{f}(\bm{x}) - \hat{f}(\bm{x}) \| &= \left\| \sum_{\bm{m}} \big(\tilde{\varphi}_{\bm{m}}(\bm{x}) - \varphi_{\bm{m}}(\bm{x}) \big) f\left(\frac{\bm{m}}{N}\right) \right\| \\[+2mm]
        &= \left\| \sum_{\bm{m}: \bm{x} \in \text{supp}\, \varphi_{\bm{m}}} \big(\tilde{\varphi}_{\bm{m}}(\bm{x}) - \varphi_{\bm{m}}(\bm{x}) \big) f\left(\frac{\bm{m}}{N}\right) \right\| \\[+2mm]
        &\leq 2^d M \max_{\bm{m}: \bm{x} \in \text{supp}\, \varphi_{\bm{m}}} |\tilde{\varphi}_{\bm{m}}(\bm{x}) - \varphi_{\bm{m}}(\bm{x})| \\[+2mm]
        &\leq 2^d M d \delta
    \end{align*}
    where $M$ is an upper bound for $f$. If we set
    \begin{equation} \label{eq:delta}
        \delta = \frac{\epsilon}{2^{d+1} M d}
    \end{equation}
    then $\| \tilde{f}(\bm{x}) - \hat{f}(\bm{x}) \| \leq \frac{\epsilon}{2}$, and hence
    \begin{equation}
        \|f - \tilde{f}\|_\infty \leq \epsilon.
    \end{equation}
    Finally, $\tilde{f}$ can be implemented as one ReLU network consisting of $(N+1)^d$ parallel subnetworks $\tilde{\varphi}_{\bm{m}}$, with a final $D$-by-$(N+1)^d$ layer of weights $f\left(\frac{\bm{m}}{N}\right)$. This results in an overall network with at most $c\ln(1/\delta)+1$ layers and $(N+1)^d(c\ln(1/\delta)+D)$ weights. With $N$ given by \eqref{eq:N} and $\delta$ given by \eqref{eq:delta}, we get our desired complexity bounds of $c \ln(1/\epsilon)+1$ layers and $c\epsilon^{-d}(\ln(1/\epsilon)+D)$ computation units.
\end{proof}

\subsection{Proof of Theorem \ref{thm:decoder_statistical}} \label{app:decoder_statistical}

Since the atlas $\{(U_i,\phi_i)\}_{i=1}^n$ is a finite distortion embedding of $\M$, each decoder $D_i = \phi_i^{-1}|_{\phi_i(U_i)} : \phi_i(U_i) \to \real^D$ is $C^k$ and Lipschitz. 
Thus, Theorem \ref{thm:decoder_statistical} holds by the following result. For this section, we assume $X \in \real^d$ has density bounded away from 0 and $\|\text{cov}(Y-f(X) | X=x)\| \leq \sqrt{c/D}$ for some constant $c$.

\begin{theorem}
    Let $f: \real[0,1]^d \to \real^D$ satisfy $f \in C^k([0,1]^d)$ and be Lipschitz with constant $L$. Given an i.i.d. sample $\{(x_i,y_i)\}_{i=1}^n$ from $(X,Y)$ satisfying the sampling assumptions above, define $N$ and the approximation $\hat{f}_n$ as in \eqref{eq:approximator_f}. Then for any $\epsilon>0$ there are $C$ and $n$ large enough so that
    \begin{equation*}
        P\big(\|f-\hat{f}_n\|_\infty > 2^d \cdot ( Cn^{-\frac{k}{2k+d}} + L\sqrt{d}N^{-1} ) \big) \leq \epsilon.
    \end{equation*}
\end{theorem}

\begin{proof}
    Let $f:\real[0,1]^d \to \real^D$ such that $f \in C^k([0,1]^d)$ and $f$ is Lipschitz with constant $L$. Let $\{\varphi_{\bm{m}}\}_{\bm{m} \in \{0,\dots,N\}^d}$ be the piecewise partition of unity introduced in \cite{yarotsky2017error} and discussed in Appendix \ref{app:decoder_complexity}. Given an i.i.d. sample $\{x_i,y_i\}_{i=1}^n$ from $(X,Y)$, at each grid point $\bm{x}_{\bm{m}} = \frac{\bm{m}}{N} \in [0,1]^d$, solve the local polynomial regression problem
    \begin{equation}
        \pi_{\bm{m},n}^*(\bm{x}) = \argmin_{\pi \in \Pi^{d \to D}_{k-1}} \frac{1}{|I_{{\bm{m}},n}|} \sum_{i \in I_{{\bm{m}},n}} \|y_i - \pi(x_i-\bm{x}_{\bm{m}})\|^2
    \end{equation}
    where $I_{{\bm{m}},n} = \{i ~|~ \|x_i-\bm{x}_{\bm{m}}\| \leq N^{-1}\}$ and $N$ satisfies $0 < \lim_{n \to \infty} n^{\frac{1}{2k+d}} N^{-1} < \infty$ as in \cite{aizenbud2021regression}. Next, let $\bar{\pi}_{\bm{m},n}$ denote the constant term of each regression function $\pi^*_{\bm{m},n}$ and define
    \begin{equation} \label{eq:approximator_f}
        \hat{f}_n(\bm{x}) = \sum_{\bm{m}} \varphi_{\bm{m}}(\bm{x}) \bar{\pi}_{\bm{m},n}.
    \end{equation}
    Then
    \begin{align*}
        \|f(\bm{x}) - \hat{f}_n(\bm{x})\| &= \Big\| \sum_{\bm{m}} \varphi_{\bm{m}}(\bm{x}) \big( f(\bm{x}) -  \bar{\pi}_{\bm{m},n} \big) \Big\| \\[+2mm]
        &\leq \sum_{\bm{m}: |x_k-\frac{m_k}{N}| < \frac{1}{N} \forall k} |\varphi_{\bm{m}}(\bm{x})| \|f(\bm{x}) -  \bar{\pi}_{\bm{m},n}\| \\[+2mm]
        &\leq \sum_{\bm{m}: |x_k-\frac{m_k}{N}| < \frac{1}{N} \forall k} \|f(\bm{x}) -  \bar{\pi}_{\bm{m},n}\| \\[+2mm]
        &\leq 2^d \max_{\bm{m}: |x_k-\frac{m_k}{N}| < \frac{1}{N} \forall k} \|f(\bm{x}) -  \bar{\pi}_{\bm{m},n}\|.
    \end{align*}
    since $\text{supp}~\varphi_{\bm{m}} = \big\{ \bm{x} : \big|x_k-\tfrac{m_k}{N}\big|<\tfrac{1}{N} ~\forall k \big\}$ and each $\bm{x}$ belongs to the support of at most $2^d$ functions $\varphi_{\bm{m}}$. Next, we write
    \begin{equation*}
        \|f(\bm{x}) -  \bar{\pi}_{\bm{m},n}\| \leq \|f(\bm{x}) -  f(\bm{x}_{\bm{m}})\| + \|f(\bm{x}_{\bm{m}}) -  \bar{\pi}_{\bm{m},n}\|.
    \end{equation*}
    For all $\bm{x} \in \text{supp}~\varphi_{\bm{m}}$, we have
    \begin{equation*}
        \|f(\bm{x}) -  f(\bm{x}_{\bm{m}})\| \leq L\|\bm{x}-\bm{x}_{\bm{m}}\| \leq \frac{L\sqrt{d}}{N}
    \end{equation*}
    since $|x_k-\frac{m_k}{N}| < \frac{1}{N} \, \forall k$ implies $\|\bm{x}-\bm{x}_{\bm{m}}\| \leq \frac{\sqrt{d}}{N}$. For the second term, given $\delta>0$ there exist $C$ and $n$ large enough so that
    \begin{equation*}
        P\big( \|f(\bm{x}_{\bm{m}}) -  \bar{\pi}_{\bm{m},n}\| > Cn^{-\frac{k}{2k+d}} \big) \leq \delta
    \end{equation*}
    by \cite{aizenbud2021regression}. Since the uncertainty only occurs at the $(N+1)^d$ grid points $\bm{x}_{\bm{m}}$, we can take a finite union bound to get
    \begin{equation*}
        P\big(\max_{\bm{x} \in \text{supp}~\varphi_{\bm{m}}}  \|f(\bm{x}) -  \bar{\pi}_{\bm{m},n}\| > Cn^{-\frac{k}{2k+d}} + L\sqrt{d}N^{-1}\big) \leq (N+1)^d \delta.
    \end{equation*}
    Choosing $C$ and $n$ large enough so that $\delta \leq \frac{\epsilon}{(N+1)^d}$, it follows that
    \begin{equation*}
        P\big(\|f-\hat{f}_n\|_\infty > 2^d \cdot ( Cn^{-\frac{k}{2k+d}} + L\sqrt{d}N^{-1} ) \big) \leq \epsilon
    \end{equation*}
    as desired. 
    
    Finally, we note that $\hat{f}_n$ can be approximated by a ReLU network similar to $\tilde{f}$ in Appendix \ref{app:decoder_complexity}, except the weights in the final layer are given by $\bar{\pi}_{\bm{m}}$, which are computed using known sample values instead of unknown function values $f(\tfrac{\bm{m}}{N})$. The parameter $N$ controls the granularity of the partition of unity and in turn the width of the network that approximates $\hat{f}_n$.
\end{proof}

\vfill

\section{Additional Numerical Results}\label{apdx:additional}

\begin{table}[h]
\begin{tabular}{cccclccc}
\hline
\multicolumn{8}{c}{\textbf{Benzine}}                                                                                                                                                                                                                                                                                                                                \\ \hline
\multicolumn{1}{l|}{Model}                                                                        & \multicolumn{1}{l|}{Latnent}             & \multicolumn{1}{l|}{N\_Chart}            & \multicolumn{1}{l|}{Z-Model} & \multicolumn{1}{l|}{\# Parm}                      & \multicolumn{1}{c|}{Recon.} & \multicolumn{1}{l|}{Energy} & \multicolumn{1}{l}{Forces} \\ \hline
\multicolumn{1}{c|}{}                                                                             & \multicolumn{1}{c|}{}                    & \multicolumn{1}{c|}{}                    & \multicolumn{1}{c|}{Linear}  & \multicolumn{1}{l|}{7417}                         & 0.319914                    & 0.7809                      & 1.0781                     \\
\multicolumn{1}{c|}{}                                                                             & \multicolumn{1}{c|}{\multirow{-2}{*}{2}} & \multicolumn{1}{c|}{}                    & \multicolumn{1}{c|}{NN}      & \multicolumn{1}{l|}{9561}                         & 0.376983                    & 0.6531                      & 0.480863                   \\ \cline{2-2} \cline{4-8} 
\multicolumn{1}{c|}{}                                                                             & \multicolumn{1}{c|}{}                    & \multicolumn{1}{c|}{}                    & \multicolumn{1}{c|}{Linear}  & \multicolumn{1}{l|}{16342}                        & 0.283381                    & 0.7733                      & 0.8733                     \\
\multicolumn{1}{c|}{}                                                                             & \multicolumn{1}{c|}{\multirow{-2}{*}{3}} & \multicolumn{1}{c|}{}                    & \multicolumn{1}{c|}{NN}      & \multicolumn{1}{l|}{18486}                        & 0.309296                    & 0.6319                      & 0.371159                   \\ \cline{2-2} \cline{4-8} 
\multicolumn{1}{c|}{}                                                                             & \multicolumn{1}{c|}{}                    & \multicolumn{1}{c|}{}                    & \multicolumn{1}{c|}{Linear}  & \multicolumn{1}{l|}{23342}                        & 0.253628                    & 0.0191                      & 0.0591                     \\
\multicolumn{1}{c|}{\multirow{-6}{*}{VAE}}                                                        & \multicolumn{1}{c|}{\multirow{-2}{*}{4}} & \multicolumn{1}{c|}{\multirow{-6}{*}{1}} & \multicolumn{1}{c|}{NN}      & \multicolumn{1}{l|}{25486}                        & 0.241641                    & 0.0123                      & 0.173903                   \\ \hline
\multicolumn{1}{c|}{}                                                                             & \multicolumn{1}{c|}{}                    & \multicolumn{1}{c|}{}                    & \multicolumn{1}{c|}{Linear}  & \multicolumn{1}{l|}{7417}                         & 0.322602                    & 0.72754                     & 0.95612                    \\
\multicolumn{1}{c|}{}                                                                             & \multicolumn{1}{c|}{\multirow{-2}{*}{2}} & \multicolumn{1}{c|}{}                    & \multicolumn{1}{c|}{NN}      & \multicolumn{1}{l|}{9561}                         & 0.393599                    & 0.608056                    & 0.47673                    \\ \cline{2-2} \cline{4-8} 
\multicolumn{1}{c|}{}                                                                             & \multicolumn{1}{c|}{}                    & \multicolumn{1}{c|}{}                    & \multicolumn{1}{c|}{Linear}  & \multicolumn{1}{l|}{16342}                        & 0.290073                    & 0.720284                    & 0.910776                   \\
\multicolumn{1}{c|}{}                                                                             & \multicolumn{1}{c|}{\multirow{-2}{*}{3}} & \multicolumn{1}{c|}{}                    & \multicolumn{1}{c|}{NN}      & \multicolumn{1}{l|}{18486}                        & 0.281581                    & 0.572098                    & 0.368218                   \\ \cline{2-2} \cline{4-8} 
\multicolumn{1}{c|}{}                                                                             & \multicolumn{1}{c|}{}                    & \multicolumn{1}{c|}{}                    & \multicolumn{1}{c|}{Linear}  & \multicolumn{1}{l|}{23342}                        & 0.229293                    & 0.019853                    & 0.059753                   \\
\multicolumn{1}{c|}{\multirow{-6}{*}{\begin{tabular}[c]{@{}c@{}}Cond.   \\      AE\end{tabular}}} & \multicolumn{1}{c|}{\multirow{-2}{*}{4}} & \multicolumn{1}{c|}{\multirow{-6}{*}{1}} & \multicolumn{1}{c|}{NN}      & \multicolumn{1}{l|}{25486}                        & 0.222649                    & 0.011073                    & 0.156688                   \\ \hline
\multicolumn{1}{c|}{}                                                                             & \multicolumn{1}{c|}{}                    & \multicolumn{1}{c|}{}                    & Linear                       & \multicolumn{1}{l|}{6824}                         & 0.283381                    & 0.75912                     & 0.871658                   \\
\multicolumn{1}{c|}{}                                                                             & \multicolumn{1}{c|}{}                    & \multicolumn{1}{c|}{\multirow{-2}{*}{2}} & NN                           & \multicolumn{1}{l|}{7944}                         & 0.309296                    & 0.609318                    & 0.429434                   \\ \cline{3-8} 
\multicolumn{1}{c|}{}                                                                             & \multicolumn{1}{c|}{}                    & \multicolumn{1}{c|}{}                    & Linear                       & \multicolumn{1}{l|}{13602}                        & 0.274378                    & 0.0373                      & 0.1119                     \\
\multicolumn{1}{c|}{}                                                                             & \multicolumn{1}{c|}{}                    & \multicolumn{1}{c|}{\multirow{-2}{*}{4}} & NN                           & \multicolumn{1}{l|}{15896}                        & 0.269803                    & 0.0235                      & 0.061145                   \\ \cline{3-8} 
\multicolumn{1}{c|}{}                                                                             & \multicolumn{1}{c|}{}                    & \multicolumn{1}{c|}{}                    & Linear                       & \multicolumn{1}{l|}{20550}                        & 0.294648                    & 0.0213                      & 0.0639                     \\
\multicolumn{1}{c|}{\multirow{-6}{*}{D+C}}                                                        & \multicolumn{1}{c|}{\multirow{-6}{*}{2}} & \multicolumn{1}{c|}{\multirow{-2}{*}{6}} & NN                           & \multicolumn{1}{l|}{23856}                        & 0.289832                    & 0.0195                      & 0.090632                   \\ \hline
\multicolumn{1}{c|}{}                                                                             & \multicolumn{1}{c|}{}                    & \multicolumn{1}{c|}{}                    & Linear                       & \multicolumn{1}{l|}{6878}                         & 0.274378                    & 0.0373                      & 0.1119                     \\
\multicolumn{1}{c|}{}                                                                             & \multicolumn{1}{c|}{}                    & \multicolumn{1}{c|}{\multirow{-2}{*}{2}} & NN                           & \multicolumn{1}{l|}{7998}                         & 0.269803                    & 0.0235                      & 0.061145                   \\ \cline{3-8} 
\multicolumn{1}{c|}{}                                                                             & \multicolumn{1}{c|}{}                    & \multicolumn{1}{c|}{}                    & Linear                       & \multicolumn{1}{l|}{13764}                        & 0.294648                    & 0.0213                      & 0.0639                     \\
\multicolumn{1}{c|}{}                                                                             & \multicolumn{1}{c|}{}                    & \multicolumn{1}{c|}{\multirow{-2}{*}{4}} & NN                           & \multicolumn{1}{l|}{{\color[HTML]{CCCCCC} 16004}} & 0.289832                    & 0.0195                      & 0.090632                   \\ \cline{3-8} 
\multicolumn{1}{c|}{}                                                                             & \multicolumn{1}{c|}{}                    & \multicolumn{1}{c|}{}                    & Linear                       & \multicolumn{1}{l|}{20658}                        & 0.227143                    & 0.0183                      & 0.05673                    \\
\multicolumn{1}{c|}{\multirow{-6}{*}{CAE}}                                                        & \multicolumn{1}{c|}{\multirow{-6}{*}{2}} & \multicolumn{1}{c|}{\multirow{-2}{*}{6}} & NN                           & \multicolumn{1}{l|}{24018}                        & 0.211678                    & 0.0079                      & 0.011169                   \\ \hline
\end{tabular}
\end{table}

\end{document}